%% file: example_paper.tex
\documentclass{article}
\usepackage{microtype}
\usepackage{graphicx}
\usepackage{subfigure}
\usepackage{booktabs} 

\usepackage{hyperref}
\usepackage{listings}
\usepackage{xcolor}
\definecolor{codegreen}{rgb}{0,0.6,0}
\definecolor{codegray}{rgb}{0.5,0.5,0.5}
\definecolor{codepurple}{rgb}{0.58,0,0.82}
\definecolor{backcolour}{rgb}{0.95,0.95,0.92}
\lstdefinestyle{mystyle}{
    backgroundcolor=\color{backcolour},   
    commentstyle=\color{codegreen},
    keywordstyle=\color{magenta},
    numberstyle=\tiny\color{codegray},
    stringstyle=\color{codepurple},
    basicstyle=\ttfamily\footnotesize,
    breakatwhitespace=false,         
    breaklines=true,                 
    captionpos=b,                    
    keepspaces=true,                 
    numbers=left,                    
    numbersep=5pt,                  
    showspaces=false,                
    showstringspaces=false,
    showtabs=false,                  
    tabsize=4,
    language=Python
}

\usepackage{multirow}

\usepackage[accepted]{icml2026}

\usepackage{amsmath}
\usepackage{amssymb}
\usepackage{mathtools}
\usepackage{amsthm}
\newcommand{\best}[1]{\underline{\textbf{\color{blue} #1}}}

\usepackage[capitalize,noabbrev]{cleveref}

\theoremstyle{plain}
\newtheorem{theorem}{Theorem}[section]

\theoremstyle{definition}

\theoremstyle{remark}

\newif\ifshowtmp
\showtmptrue

\NewDocumentCommand{\Distance}{ O{} }{  {\mathbb{D}}_{\rm #1}  }

\NewDocumentCommand{\constraint}{ O{} O{t} }{ {\mathcal{C}}^{\rm #1}_{#2}  }

\NewDocumentCommand{\loss}{ O{} }{  \mathcal{L}^{\rm #1}  }

\NewDocumentCommand{\ratio}{ O{t} }{ w_{#1}  } 

\NewDocumentCommand{\KL}{ O{t} }{  {\rm KL}_{#1} } 
\NewDocumentCommand{\KLThree}{ O{t} }{  {\rm KL3}_{#1} } 
\NewDocumentCommand{\KLThreeText}{ }{KL3 }

\NewDocumentCommand{\lowerClippingRange}{ O{} O{}  }{  l_{#1}^{#2}  } 
\NewDocumentCommand{\upperClippingRange}{ O{} O{} }{  u_{#1}^{#2}  }

\input{lib_custom/A_ALL}

\ifshowtmp
    \newcommand{\qingyuan}[1]{\textcolor{blue}{(\textbf{Qingyuan}: #1)}}
    \newcommand{\yuhui}[1]{\textcolor{teal}{(\textbf{Yuhui}: #1)}}
    \newcommand{\simon}[1]{\textcolor{orange}{(\textbf{Simon:} #1)}}
    \newcommand{\shilong}[1]{\textcolor{blue}{(\textbf{Shilong}: #1)}}
\else
    \newcommand{\qingyuan}[1]{}
    \newcommand{\yuhui}[1]
    \newcommand{\simon}[1]
    \newcommand{\shilong}[1]
\fi

\icmltitlerunning{A Unified Framework for Rethinking Policy Divergence Measures in GRPO}

\begin{document}

\twocolumn[
\icmltitle{A Unified Framework for Rethinking Policy Divergence Measures in GRPO}



  \icmlsetsymbol{equal}{*}

\begin{icmlauthorlist}
\icmlauthor{Qingyuan Wu}{southampton,cohere,equal}
\icmlauthor{Yuhui Wang}{kaust,equal}
\icmlauthor{Simon Sinong Zhan}{northwestern}
\icmlauthor{Yanning Dai}{kaust}
\icmlauthor{Shilong Deng}{liverpool}
\icmlauthor{Sarra Habchi}{cohere}
\icmlauthor{Qi Zhu}{northwestern}
\icmlauthor{Matthias Gallé}{cohere}
\icmlauthor{Chao Huang}{southampton}
\end{icmlauthorlist}

\icmlaffiliation{cohere}{Cohere}
\icmlaffiliation{southampton}{University of Southampton}
\icmlaffiliation{liverpool}{University of Liverpool}
\icmlaffiliation{kaust}{KAUST}
\icmlaffiliation{northwestern}{Northwestern University}

\icmlcorrespondingauthor{Qingyuan Wu}{qingyuan.wu@soton.ac.uk}
\icmlkeywords{Machine Learning, ICML}

\vskip 0.3in
]

\printAffiliationsAndNotice{\icmlEqualContribution}

\begin{abstract}
Reinforcement Learning with Verified Reward (RLVR) has emerged as a critical paradigm for advancing the reasoning capabilities of Large Language Models (LLMs). Most existing RLVR methods, such as GRPO and its variants, ensure stable updates by constraining policy divergence through clipping likelihood ratios. This paper introduces a unified clipping framework that characterizes existing methods via a general notion of policy divergence, encompassing both likelihood ratios and Kullback-Leibler (KL) divergences and extending to alternative measures. The framework provides a principled foundation for systematically analyzing how different policy divergence measures affect exploration and performance. We further identify the KL3 estimator, a variance-reduced Monte Carlo estimator of the KL divergence, as a key policy divergence constraint. We theoretically demonstrate that the KL3-based constraint is mathematically equivalent to an asymmetric ratio-based clipping that reallocates probability mass toward high-confidence actions, promoting stronger exploration while retaining the simplicity of GRPO-style methods. Empirical results on mathematical reasoning benchmarks demonstrate that incorporating the KL3 estimator into GRPO improves both training stability and final performance, highlighting the importance of principled policy divergence constraints in policy optimization.
\end{abstract}

\section{Introduction}\label{sec:introduction}
Reinforcement Learning (RL)~\citep{kaelbling1996reinforcement, rlai} has served as a pivotal training paradigm in decision-making problems~\citep{tesauro1994td, silver2016mastering, mnih2013playing, berner2019dota}, and has recently been playing a central role in advancing Large Language Models (LLMs)~\citep{ouyang2022training, lambert2024tulu}.
RL provides an efficient and general training framework for LLMs, enabling optimization over complex, non-differentiable objectives that extend beyond direct supervised learning from human data.
This capability is particularly critical for real-world tasks such as code generation~\citep{jain2024livecodebench}, mathematical reasoning~\citep{zhang2024careful, cobbe2021training}, and dialogue alignment~\citep{chiang2024chatbot}.

Current RL methodologies for LLMs, particularly in RL with Verified Reward (RLVR)~\citep{lambert2024tulu} settings, predominantly rely on Proximal Policy Optimization (PPO)~\citep{proximal_policy_optimization}.
PPO ensures training stability through the ratio-based clipping mechanism, aiming to approximate the trust-region constraint of Trust Region Policy Optimization (TRPO)~\citep{trpo}.
Recently, Group Relative Policy Optimization (GRPO)~\citep{shao2024deepseekmath} has emerged as a memory-efficient alternative for training large-scale LLMs by using group-normalized returns as the advantage baselines, thus eliminating the need to maintain a separate value function.
Like PPO, GRPO and its variants~\citep{yu2025dapo, yang2025dcpo} rely on the ratio-based clipping to ensure stable policy updates.

Despite their success, the reliance on ratio-based clipping constitutes a specific and potentially restrictive design choice within the broader landscape of policy optimization. 
While these methods aim to ensure stable updates by constraining policy divergences through clipping likelihood ratios, recent studies~\citep{cui2025entropy, park2025clip} reveal that both training exploration and evaluation performance are highly sensitive to the specific definition and implementation of policy divergence constraints.
Although different variants of the ratio-based clipping mechanism~\citep{yu2025dapo, yang2025dcpo} have been proposed, a principled understanding of how different policy divergence measures and corresponding constraints affect the trade-off between exploration and stability remains largely unexplored.

To address this issue, this paper first introduces a unified clipping framework that characterizes existing clipping methods under a general notion of policy divergence. This unified clipping framework provides a foundational perspective for analyzing various policy divergence constraints, encompassing both likelihood ratios and Kullback--Leibler (KL) divergences.
Furthermore, we identify that the $\KLThree[]$ estimator~\citep{schulman2020approximating} serves as the pivotal policy divergence constraint under our framework.
We theoretically demonstrate that the $\KLThree[]$-based constraint is mathematically equivalent to an asymmetric ratio-based clipping.
Based on these observations, we propose Approximate Trust Region-based GRPO (ATR-GRPO).
Unlike standard symmetric ratio-based clipping, ATR-GRPO leverages the \KLThreeText estimator to actively steer exploration by reallocating probability mass toward high-confidence actions, all while maintaining the computational efficiency of GRPO.
Comprehensive experiments on mathematical reasoning benchmarks demonstrate that \KLThreeText estimator consistently improves training stability and final performance compared to existing state-of-the-art (SOTA) baselines.
Our contributions can be summarized as follows:
\begin{itemize}
    \item We introduce a unified clipping framework for policy optimization that unifies existing policy divergence constraints and can be extended to arbitrary alternative measures.
    \item We identify the KL3 estimator as an effective policy divergence constraint, establish its connection to ratio-based constraints, and theoretically show that it promotes stronger exploration than existing alternatives.
    \item Building on these insights, we develop ATR-GRPO, which exhibits improved exploration dynamics while retaining the simplicity of existing methods.
    \item We empirically show that ATR-GRPO enhances learning stability and achieves performance competitive with various SOTA baselines.
\end{itemize}

\section{Related Work}\label{sec:related_works}
\paragraph{RL for LLMs.}
Reinforcement Learning (RL) has been established as the standard paradigm for advancing Large Language Models (LLMs), such as RLVR~\citep{lambert2024tulu}.
PPO~\citep{proximal_policy_optimization} utilizes a clipped surrogate objective to ensure stable policy updates. However, PPO requires training a separate critic model, which can be prohibitive for large-scale reasoning tasks. 
To mitigate these memory constraints, RLOO~\cite{ahmadian2024back} eliminates the critic by employing the Leave-One-Out baseline to reduce variance by averaging rewards across other samples within the batch.
GRPO~\citep{shao2024deepseekmath} calculates advantages relative to a sampled group of outputs for each prompt, effectively optimizing memory usage for reasoning tasks.
GSPO~\citep{zheng2025group} further extends this paradigm by elevating the optimization granularity to the sequence level, aligning updates with reward signals.
Building on GRPO or GSPO, some recent works such as GTPO~\citep{tan2025gtpo} and EMPO~\citep{zhang2025right} incorporate semantic entropy into reward shaping to address the persistent challenge of sparse credit assignment.

\paragraph{Trust Region Methods in RL.}
Trust-region methods underpin stable policy optimization in RL, from NPG~\citep{kakade2001natural} to the explicit KL-constrained formulation of TRPO~\citep{trpo}.
However, TRPO relies on computationally expensive second-order optimization, making it impractical for large-scale RL tasks. 
To address this scalability issue, PPO~\citep{proximal_policy_optimization} was introduced as a first-order approximation that replaces the hard KL constraint with a ratio-based clipping, aiming to implicitly constrain policy updates. While PPO has become the de facto choice for large-scale training due to its computational efficiency, its reliance on ratio-based clipping does not explicitly enforce a trust-region constraint. Prior work has shown that this approximation can fail to properly control policy updates, leading to potential optimization instability, as demonstrated by TRGPPO and Truly PPO~\citep{wang2019trust, wang2020truly}.
In this work, we propose a unified clipping framework that generalizes different policy divergence constraints, including likelihood ratios and KL divergences.

\paragraph{Clipping Mechanisms in RL.}
TRGPPO~\citep{wang2019trust} shows that ratio-based clipping in PPO overly constrains low-likelihood probabilities, and introduces dynamic clipping to relax this effect.
Truly PPO~\citep{wang2020truly} further exposes the divergence between ratio-based and KL-based constraints, and proposes a KL-based clipping method as an alternative.
DAPO~\citep{yu2025dapo} introduces an asymmetric clip-higher mechanism that increases the upper clipping range, mitigating entropy collapse and facilitating probability increases for low-likelihood exploratory tokens.
DCPO~\citep{yang2025dcpo} proposes the dynamic ratio-based clipping mechanism, which can adaptively adjust the clipping ranges based on the prior probabilities to strengthen exploration ability.
Aside from exploration purposes, the dual clipping mechanism~\citep{ye2020mastering} is proposed to ensure the convergence and stability in large-scale distributed RL training of MOBA games, which is also adopted as the default technique in the LLM training framework~\citep{sheng2025hybridflow}.
SAPO~\citep{gao2025soft} presents the soft gate operation to replace the hard clipping, stablizing the optimization. 
These ratio-based clipping mechanisms, however, designed for stability, exert a profound and often detrimental influence on policy entropy~\citep{cui2025entropy, park2025clip}.
Different from previous approaches, this paper identifies the KL3 estimator~\citep{schulman2020approximating} as the pivotal policy divergence constraint within our proposed framework, demonstrated by our theoretical analysis and empirical evaluation.

\begin{table*}[t]
\centering
\caption{
Comparative analysis of policy divergence constraints. Unlike previous methods that rely on either heuristic symmetric ratio-based clipping (PPO, GRPO), asymmetric ratio-based clipping (DAPO) or computationally expensive full expectation of the KL divergence (TRPO, Truly PPO), our ATR-GRPO achieves a principled, approximate trust-region constraint with low computational cost.
\label{table::clipping_comparison}}
\resizebox{\textwidth}{!}{%
\begin{tabular}{lccc}
\hline\hline
Method  & Clipping / Constraint Criterion                                                                                                                                     & Trust-Region Constraint? & Computational Cost \\ \hline
PPO~\citep{proximal_policy_optimization}, GRPO~\citep{shao2024deepseekmath}             & $\ratio(\theta) \in [1 - \epsilon, 1 + \epsilon]$                                                                                                          & $\times$                 & Low                     \\
DAPO~\citep{yu2025dapo}      & $\ratio(\theta) \in [1 - \epsilon_l, 1 + \epsilon_u]$                                                                                             & $\times$                 & Low                      \\
DCPO~\citep{yang2025dcpo} & $\ratio(\theta) \in [0.5 + \frac{1}{2}\sqrt{\max(1 - \frac{4\epsilon_l}{\pi_{\theta_\text{old}}})}, 0.5 + \frac{1}{2}\sqrt{1 + \frac{4\epsilon_u}{\pi_{\theta_\text{old}}}}]$ & $\times$                 & Low                   \\
TRPO~\citep{trpo}          & $\KL(\theta) \leq \delta$                                                             & $\checkmark$             & High                  \\
TRGPPO~\citep{wang2019trust}          & $\begin{aligned}
\ratio(\theta) \in \bigl[\min\!\bigl(l_\delta^\text{KL},\, 1-\epsilon\bigr),\max\!\bigl(u_\delta^\text{KL},\, 1+\epsilon\bigr)
\bigr]
\end{aligned}$                                                             & $\checkmark$             & Medium                      \\
Truly PPO~\citep{wang2020truly}          & $\KL(\theta) \leq \delta$                                                             & $\checkmark$             & Medium (for large action space $|\mathcal{A}|$)                   \\
ATR-GRPO (ours)   & $\KLThree(\theta) \leq \delta \text{ or } \ratio(\theta) \in [\lowerClippingRange[\delta][\KLThree[]], \upperClippingRange[\delta][\KLThree[]]]$                                                          & $\checkmark$ (approximated by \KLThreeText)      & Low                      \\ \hline\hline
\end{tabular}
}
\end{table*}

\section{Preliminaries}\label{sec:preliminaries}
\paragraph{Language Model Generation as MDP.}
Language generation process can be formulated as a Markov Decision Process (MDP), denoted by the tuple $(\mathcal{S}, \mathcal{A}, \mathcal{T}, r)$.
Here, $\mathcal{S}$ denotes the state space, where the state at timestep $t$ is defined as $s_t \triangleq (x, y_{<t})$, corresponding to the concatenation of the input query $x$ and the partially generated response $y_{<t}$.
The action space $\mathcal{A}$ is defined over the token vocabulary, $\mathcal{T}$ represents the transition dynamics induced by autoregressive token generation, and the reward function $r$ is defined over concatenations of the input query $x$ and the partially generated response.
The policy $\pi_\theta$ parameterized by $\theta$ aims to maximize the objective $J(\theta)$ defined as:
\begin{fontsize}{9}{1}
\begin{equation}\label{eq::objective}
\begin{aligned}
J(\theta)=  \underset{x\sim \mathcal{D}}{\mathbb{E}} & \Bigg[\underset{}{\mathbb{E}}^{}\left[ \sum\nolimits_{t} r( s_{t} ,y_{t}) ; \pi_{\theta} \right] \\ &\quad -\beta \ \mathrm{KL}\big( 
 \pi _{\theta } (\cdot | s_{t}) || \pi _{\mathrm{ref}} (\cdot |s_{t})\big)
 \Bigg],\\
\end{aligned}
\end{equation}
\end{fontsize}\noindent\ignorespaces
where $\mathcal{D}$ denotes the query distribution, and $y$ is the response generated by the language model $\pi_\theta$.
$\pi_{\text{ref}}$ is a reference policy, $\KL[]$ denotes the Kullback--Leibler (KL) divergence between two policies (e.g., $\pi_1$ and $\pi_2$), defined as follows:
\begin{fontsize}{9}{1}
\begin{equation}\label{eq_KL}
    \KL[]
        \big( \pi_1(\cdot|s) || \pi_2(\cdot|s) \big)
    \triangleq
    \sum_{a \in \mathcal{A}}
    \pi_1(a|s)
    \log \frac{\pi_1(a|s)}{\pi_2(a|s)},
\end{equation}
\end{fontsize}\noindent\ignorespaces
and $\beta$ controls the KL regularization to avoid policy drift~\citep{shao2024deepseekmath}.


\paragraph{Clipping in Policy Optimization Algorithms.}
Given a query $x$ and a generated response $y$, the clipped surrogate objective for the policy $\pi_\theta$ is defined as:
$$
\mathcal{L} (\theta)=
\mathop{\mathbb{E}}_{
    y \sim \pi_\theta(\cdot|x) \atop
    x \sim \mathcal{D}
}
\left[
\min\left( w_{t}(\theta )A_{t} ,\ {\mathrm{clip}}\left( w_{t} (\theta ),\ \cdot \ \right)A_{t}\right)
\right],
$$
where $w_t(\theta) = \frac{\pi_\theta(y_t \mid s_t)}{\pi_{\theta_{\text{old}}}(y_t \mid s_t)}$ represents the token-level likelihood ratio between $\pi_\theta$ and old policy $\pi_{\theta_{\text{old}}}$.
Here, the term $A_t$ denotes the advantage estimate, which corresponds to a group-normalized score in GRPO~\citep{shao2024deepseekmath} and the Generalized Advantage Estimator~\citep{schulman2015high} in PPO~\citep{proximal_policy_optimization}.
The clipping function $\mathrm{clip}(\cdot)$ constrains the likelihood ratio, thereby preventing excessively large policy updates and ensuring training stability.
Without loss of generality, we adopt the generalized clipping formulation proposed by \citep{wang2019trust}:
\begin{fontsize}{9.9pt}{1pt}
\begin{equation}
\mathrm{clip}^{\text{ratio}}\left( w_{t} (\theta ),l_{t}^{\text{}} ,u_{t}^{\text{}}\right) =\begin{cases}
w_{t} (\theta ), & l_{t}^{\text{}} \leq w_{t} (\theta )\leq u_{t}^{\text{}}\\
l_{t}^{\text{}} , & w_{t} (\theta )\leq l_{t}^{\text{}}\\
u_{t}^{\text{}} , & w_{t} (\theta )\geq u_{t}^{\text{}}
\end{cases}
\end{equation}
\end{fontsize}\ignorespacesafterend
where $\lowerClippingRange[t]$ and $\upperClippingRange[t]$ denote the lower and upper clipping ranges, respectively.
PPO and GRPO both use the symmetric clipping ranges $(l_t = 1 - \epsilon, u_t = 1 + \epsilon)$, while DAPO~\cite{yu2025dapo} uses the asymmetric clipping ranges $(l_t = 1 - \epsilon_l, u_t = 1 + \epsilon_u)$.

Despite the popularity of ratio-based clipping, \citeauthor{wang2020truly} identified a mismatch between the ratio-based and the KL-based constraints, and accordingly proposed an alternative KL-based clipping function:
\begin{equation}\label{eq_KL_based_clipping}
\mathrm{clip}^{\text{KL}}( \ratio (\theta ), \delta ) =
\begin{cases}
\ratio (\theta ), & \KL( \theta ) \leq \delta, \\
\ratio (\theta_{\text{old}} ), & \text{otherwise},
\end{cases}
\end{equation}
where $\delta$ is the trust-region threshold and $\KL(\theta) \triangleq  \KL\big( \pi_\theta( \cdot | s_t ) || \pi_{\theta_\text{old}}( \cdot | s_t ) \big) $ involves an expectation over the full action space (\Cref{eq_KL}). 
We distinguish loss functions induced by different clipping functions using subscripts.
For instance, the objectives corresponding to ratio-based and KL-based clipping are denoted by $\loss[ratio]$ and $\loss[KL]$, respectively.

\section{Method}\label{sec:our_method}

\subsection{A Unified Clipping Framework}

We first propose a unified framework that subsumes both ratio-based and KL-based policy constraints, while naturally accommodating more general constraint variants.
Specifically, we define a general clipping operator as
\begin{equation}\label{eq_general_clipping}
\mathrm{clip}_{\text{general}}( \ratio (\theta ),\constraint[][]) =\begin{cases}
\ratio (\theta ), & \text{if } \constraint( \theta ) \text{ is true}\\
\ratio (\theta _{\text{old}} ), & \text{otherwise}
\end{cases}
\end{equation}
where the constraint function $\constraint(\theta)$ encodes a prescribed feasibility condition on the policy corresponding to the sample $(s_t, a_t)$.
This formulation provides a flexible abstraction of policy divergence constraints: rather than committing to a specific measure, the operator admits arbitrary constraint functions that define how policy divergences are measured and restricted at the sample level.

Under this unified clipping framework, existing methods can be recovered as special cases by specifying different choices of the constraint function, as illustrated in~\Cref{table::clipping_comparison}.
In particular, when the constraint is set as
$\constraint[KL](\theta) \coloneqq \KL(\theta) \leq \delta$,
the resulting surrogate objective exactly recovers the original KL-constrained formulation, i.e.,
$\loss[KL]_{\rm general}(\theta) = \loss[KL](\theta)$, corresponding to an explicit trust-region constraint.
When the constraint is set as
$\constraint[ratio](\theta) \coloneqq l_t \leq \ratio \leq u_t$,
the resulting surrogate objective yields the same gradient as the standard ratio-based clipping objective.
As a result, this choice implicitly constrains policy updates without explicitly enforcing a trust-region constraint.
We formalize this equivalence in the following theorem.

\begin{theorem}[Gradient Equivalence]
\label{thm:gradient_equivalence}
Let the constraint be defined as $\constraint[ratio](\theta) \coloneqq l_t \leq \ratio(\theta) \leq u_t$. Then, for any parameter $\theta$ where the objective is differentiable, the gradient of the general objective is equivalent to that of the ratio-based objective:
$
    \nabla \loss[ratio](\theta) = \nabla \loss[ratio]_{\rm general}(\theta).
$
\end{theorem}

\subsection{Analysis of $\KLThree[]$ Estimator and ATR-GRPO}
We have shown that both ratio-based and KL-based objectives are special cases of the unified formulation in \Cref{eq_general_clipping}.
Next, we investigate other policy divergence constraints that can be incorporated within this framework.
Previous literature has shown that the KL-based constraint
$
\KL[](\theta) \leq \delta
$ 
aligns with the original trust-region theory, which enforces monotonic policy improvement~\citep{trpo, wang2020truly, wang2019trust}.
However, in the context of LLMs, enforcing such a KL-based constraint is often intractable, as it requires computing the full expectation of the KL divergence over an extremely large action space.

To address this limitation, \citeauthor{schulman2020approximating} proposed several Monte-Carlo estimators to approximate the KL divergence.
Among them, the most widely used estimator is the $\KLThree$ operator, which admits a lightweight surrogate expression:
\begin{equation}\label{eq:kl3}
    \KLThree(\theta) \coloneqq \ratio(\theta) - 1 - \log \ratio(\theta).
\end{equation}
This estimator offers several appealing properties.
First, it can be computed at the sample level without requiring an explicit expectation over the entire action space, which is crucial for large-scale policies such as LLMs.
Second, $\KLThree$ is non-negative and exhibits substantially lower variance than naive Monte-Carlo KL estimators.
Finally, $\KLThree$ provides a local approximation to the KL divergence near the identity ratio, and thus serves as a principled trust-region surrogate~\citep{schulman2020approximating}.

Below, we analyze the relationship between different policy divergence constraints within the unified framework.
As implied by \Cref{eq:kl3}, $\KLThree[t](\theta)$ can be viewed as a function of the ratio $\ratio(\theta)$, and we show that the two constraints (likelihood ratio and KL divergence) are equivalent under appropriate hyperparameter choices.

\begin{theorem}[Equivalence and Asymmetry]\label{theorem_KL3}
Define lower and upper clipping ranges for a given threshold $\delta > 0$ as:
\[
\lowerClippingRange[\delta][\KLThree[]] = \min_{ \theta } \ratio(\theta) \quad \text{s.t.} \quad \KLThree(\theta) \leq \delta,
\]
\[
\upperClippingRange[\delta][\KLThree[]] = \max_{ \theta } \ratio(\theta) \quad \text{s.t.} \quad \KLThree(\theta) \leq \delta.
\]
(1) The constraint $\KLThree(\theta) \leq \delta$ is equivalent to $\lowerClippingRange[\delta][\KLThree[]] \leq \ratio(\theta) \leq \upperClippingRange[\delta][\KLThree[]]$, where $0 < \lowerClippingRange[\delta][\KLThree[]] < 1 < \upperClippingRange[\delta][\KLThree[]]$.

(2) These ranges satisfy the asymmetry property: $1 - \lowerClippingRange[\delta][\KLThree[]] < \upperClippingRange[\delta][\KLThree[]] - 1$.
\end{theorem}

The proof of \Cref{theorem_KL3} is provided in~\Cref{appendix:theorem_KL3}.
\Cref{theorem_KL3} reveals an explicit connection between the KL3-based and ratio-based constraints by characterizing the admissible likelihood ratio ranges under the KL3 constraint, in the spirit of TRGPPO~\citep{wang2019trust}.
\textbf{Importantly, our goal here is to reveal this underlying relationship between these two types of constraints, instead of how to compute these clipping ranges explicitly.}
In practice, it is not necessary to explicitly compute the clipping ranges, but can directly restrict policy divergences through the specified clipping function $\constraint(\theta)$ in \Cref{eq_general_clipping}.

\begin{figure}[h]
    \centering
    \includegraphics[width=\linewidth]{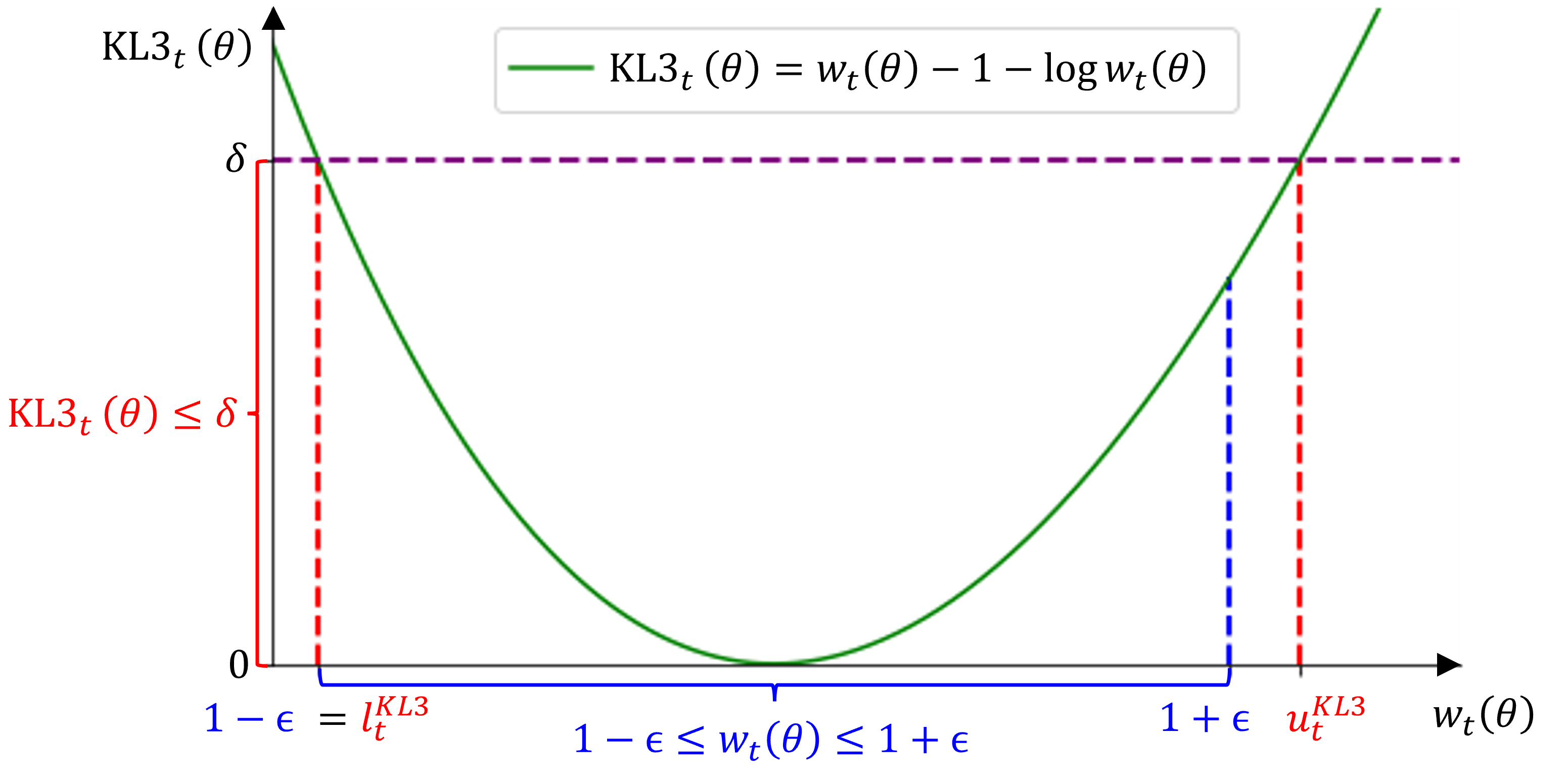}
    \caption{
    Illustration of the KL3-based constraint.
    \label{fig:clip_analysis_high}
    }
\end{figure}

As implied by \Cref{theorem_KL3} (1), the KL3-based constraint is equivalent to a ratio-based constraint with a specific choice of hyperparameters.
Furthermore, as implied by \Cref{theorem_KL3} (2), the KL3-based constraint always induces an asymmetric clipping range, where the upper clipping range deviates more than the lower range does ($1 - \lowerClippingRange[\delta][\KLThree[]] < \upperClippingRange[\delta][\KLThree[]] - 1$). 
Also, as illustrated in the \Cref{fig:clip_analysis_high}, the constraint $\KLThree(\theta) \leq \delta$ results in a larger upper clipping range compared to the symmetric clipping range $(1-\epsilon, 1+\epsilon)$, where we set $\epsilon = 1 - \lowerClippingRange[\delta][\KLThree[]]$.
This property is consistent with prior work, such as Clip-Higher \cite{yu2025dapo}.
However, unlike heuristic asymmetric clipping rules of Clip-Higher, the KL3-based constraint provides an exact and principled characterization of this asymmetry, thereby offering a theoretically grounded mechanism for tuning the clipping range.


Building upon these observations, we formally propose Approximate Trust Region-based GRPO (ATR-GRPO), which integrates the $\KLThree[]$ constraint directly into the unified framework(\Cref{eq_general_clipping}) by setting $\constraint(\theta):= \KLThree(\theta) \leq \delta$.
Different from GRPO, ATR-GRPO enforces the approximate trust-region constraint through the ATR-based clipping ($\ratio(\theta) \in [\lowerClippingRange[\delta][\KLThree[]], \upperClippingRange[\delta][\KLThree[]]]$) within the unified framework (\Cref{eq_general_clipping}), enabling stable policy updates without explicitly computing the full KL expectations.

\section{Theoretical Analysis}
In this section, we present a comprehensive theoretical analysis of ATR-based clipping and ratio-based clipping, building on the frameworks developed in prior works~\citep{cui2025entropy, park2025clip}. 
We formally characterize the induced policy logit differences (\Cref{theorem:performance_diff}) and analyze their impact on policy entropy difference (\Cref{theorem:entropy_diff}).

\subsection{Setup and Assumptions}
We assume that GRPO (symmetric ratio-based clipping) and ATR-GRPO (ATR-based clipping) both utilize the full-batch gradients under the standard policy gradient algorithm~\citep{williams1992simple}. 
Furthermore, the policy $\pi_{\theta}$ is modeled as a softmax policy $\pi_{\theta}(a|s) = \frac{\exp(\theta_{s, a})}{\sum_{a' \in \mathcal{A}(s)} \exp(\theta_{s, a})}$.
We mainly consider two representative cases, and we denote the corresponding probabilistic events as:
\begin{fontsize}{9}{1}
\begin{equation*}
\begin{aligned}
&X_{-}(s):=\{a\in \mathcal{A}(s)| \ratio(\theta_{s,a}) \in [1 - \epsilon, \lowerClippingRange[\delta][\KLThree[]]], 1 + \epsilon = \upperClippingRange[\delta][\KLThree[]]\},\\
&X_{+}(s):=\{a\in \mathcal{A}(s)| \ratio(\theta_{s,a}) \in [1 + \epsilon, \upperClippingRange[\delta][\KLThree[]]], 1 - \epsilon = \lowerClippingRange[\delta][\KLThree[]]\},\\
\end{aligned}
\end{equation*}
\end{fontsize}\ignorespacesafterend
where $X_{-}(s)$ and $X_{+}(s)$ represent the event of unsatisfying and satisfying the KL3 constraint (approximate trust-region constraint), respectively.
We define $\mathbb{I}_{X_{-}(s)}(a)$ and $\mathbb{I}_{X_{+}(s)}(a)$ to be the indicator function for $X_{-}(s)$ and $X_{+}(s)$, respectively.

\subsection{Policy Logits Difference Analysis}
Let $\theta^{k}_{s, a}$ denote the policy logits at the $k$-th step. We denote the updated policy logits for the ratio-based and ATR-based clipping methods as $\theta^{\text{ratio,}k+1}_{s, a}$ and $\theta^{\text{ATR}, k+1}_{s, a}$, respectively.
We define $\Delta \theta_{s, a}:= \theta^{\text{ATR}, k+1}_{s, a} - \theta^{\text{ratio,}k+1}_{s, a}$ as the policy logits difference between ATR-based clipping and the ratio-based clipping.
We characterize $\Delta \theta_{s, a}$ under different events in the following~\Cref{theorem:performance_diff} and provide proof in~\Cref{appendix:proof:performance_diff}.

\begin{theorem}[Policy Logits Difference]\label{theorem:performance_diff}
Consider the policy gradient algorithm with learning rate $\eta$ and given the state visitation distribution $d^{\pi_{\theta_\text{old}}(s)}$ induced by $\pi_{\theta_\text{old}}$. 
Let $\Delta \theta_{s, a}$ denote the policy logits difference between the ATR-based and ratio-based clipping methods. 

Under event $X_{-}(s)$, we have 
$$
\begin{aligned}
    \Delta \theta_{s, a} = - \eta d^{\pi_{\theta_\text{old}}(s)}\pi_{\theta^k}(a|s)\bigg[&A\mathbb{I}_{X_{-}(s)}(a) )\\
    &
        - 
        \mathop{\mathbb{E}}_{a' \sim \pi_{\theta^k}(\cdot|s)}
        [A\mathbb{I}_{X_{-}(s)}(a')]
    \bigg].\\
\end{aligned}
$$
Under event $X_{+}(s)$, we have
$$
\begin{aligned}
    \Delta \theta_{s, a} = \eta d^{\pi_{\theta_\text{old}}(s)}\pi_{\theta^k}(a|s)\bigg[&A\mathbb{I}_{X_{+}(s)}(a) \\
    & 
            - 
            \mathop{\mathbb{E}}_{a' \sim \pi_{\theta^k}(\cdot|s)}
            [A\mathbb{I}_{X_{+}(s)}(a')]
        \bigg].\\
\end{aligned}
$$
\end{theorem}

\subsection{Entropy Difference Analysis}
Building upon~\Cref{theorem:performance_diff}, we characterize the exploration behaviours of ATR-based clipping by deriving the entropy difference $\Delta \mathcal{H}:= \mathcal{H}(\theta^{\text{ATR}, k+1} | s) - \mathcal{H}(\theta^{\text{ratio,}k+1} | s)$ in the following theorem.
The proof is provided in~\Cref{appendix:proof:entropy_diff}.

\begin{theorem}[Entropy Difference]\label{theorem:entropy_diff}
The entropy difference $\Delta \mathcal{H}$ between the ATR-based and ratio-based clipping methods is related to the covariance between the advantage and the log-likelihood.

Under event $X_{-}(s)$, we have
$$
    \Delta \mathcal{H} = 
    \eta d^{\pi_{\theta_\text{old}}(s)}\mathop{Cov}_{a \sim \pi_{\theta^k}(\cdot|s)}\left(A\mathbb{I}_{X_{-}(s)}(a), \log \pi_{\theta^{k}}(a|s) \right).
$$
Under event $X_{+}(s)$, we have
$$
    \Delta \mathcal{H} = 
    - \eta d^{\pi_{\theta_\text{old}}(s)}\mathop{Cov}_{a \sim \pi_{\theta^k}(\cdot|s)}\left(A\mathbb{I}_{X_{+}(s)}(a), \log \pi_{\theta^{k}}(a|s) \right).
$$
\end{theorem}

\Cref{theorem:entropy_diff} implies that the entropy difference depends on the covariance between the advantage under the specific event and the log-likelihood of $\pi_{\theta^{k}}$.
Crucially, this dependency reveals that ATR-based clipping exhibits exploration behavior that is better aligned with the approximate trust-region constraint than ratio-based clipping.
We identify this as the dynamic exploration mechanism: \textbf{it enforces conservative exploration when updates are large and potentially risky ($X_{-}(s)$) while permitting expansive exploration when updates are within the approximate trust-region ($X_{+}(s)$)}.

\begin{table*}[t]
\centering
\caption{Performance comparison on the AMC2023, AIME2024, and AIME2025 benchmarks. We report both the final and best evaluation performance (formatted as final (best)). The best performance is highlighted.\label{table:qwen3_performance}}
\resizebox{0.97\textwidth}{!}{%
\begin{tabular}{lcccccccc}
\toprule
\toprule
\multirow{2.5}{*}{\textbf{Method}} & \multicolumn{2}{c}{AMC2023} & \multicolumn{2}{c}{AIME2024} & \multicolumn{2}{c}{AIME2025} & \multicolumn{2}{c}{Average} \\
\cmidrule(lr){2-3} \cmidrule(lr){4-5} \cmidrule(lr){6-7} \cmidrule(lr){8-9}
 & Mean@8 (\%) & Pass@8 (\%) & Mean@8 (\%) & Pass@8 (\%) & Mean@8 (\%) & Pass@8 (\%) & Mean@8 (\%) & Pass@8 (\%) \\
\midrule
\multicolumn{9}{c}{\textbf{\textit{Qwen3-1.7B}}} \\
\midrule
Base Model       & 28.61 & 48.19 & 6.67  & 13.33 & 4.17  & 20.00 & 13.15 & 27.18 \\
Clip
& 40.21 (40.21) & 60.24 (67.47) & 10.00 (10.00) & 23.33 (26.67) & 10.42 (12.50) & 20.00 (26.67) & 20.21 (20.41) & 34.52 (36.93) \\
Clip-Higher
& 37.65 (37.80) & 66.27 (66.27) & 10.42 (10.83) & 23.33 (30.00) & 6.25 (10.42)  & 16.67 (26.67) & 18.11 (18.68) & 35.42 (37.15) \\
Dual Clip
& 42.02 (42.02) & 65.06 (68.67) & 10.83 (11.67) & \best{30.00} (33.33) & 12.50 (13.33) & 23.33 (26.67) & 21.78 (21.78) & 39.46 (39.56) \\
Dynamic Clipping
& 41.72 (41.72) & 63.86 (67.47) & 9.17 (11.67)  & 20.00 (30.00) & 11.67 (12.92) & 23.33 (26.67) & 20.85 (21.56) & 35.73 (38.35) \\
Clip-Cov
& 33.58 (36.14) & 61.45 (67.47) & 7.92 (9.58)   & 20.00 (23.33) & 7.50 (10.83)  & 16.67 (23.33) & 16.33 (18.80) & 32.70 (34.93) \\
Soft Gate
& \best{42.92} ({42.92}) & {68.67} ({68.67}) & 8.33 (9.58)   & 26.67 (30.00) & 13.33 ({13.33}) & 23.33 (26.67) & 21.53 (21.53) & 39.56 (39.87) \\
ATR-based Clipping (ours)
& 41.72 (\best{45.03}) & \best{69.88} (\best{72.29}) & \best{13.33} (\best{13.33}) & {26.67} (\best{36.67}) & \best{13.75} (\best{14.58}) & \best{30.00} (\best{30.00}) & \best{22.93} (\best{23.07}) & \best{42.18} (\best{44.00}) \\
\midrule
\multicolumn{9}{c}{\textbf{\textit{Qwen3-8B}}} \\
\midrule
Base Model       & 26.05 & 50.60 & 5.42  & 16.67 & 1.25  & 6.67  & 10.91 & 24.65 \\
Clip
& 45.78 (48.34) & 72.29 (74.70) & 12.92 (18.33) & 26.67 (36.67) & 15.83 (15.83) & 30.00 (33.33) & 24.84 (26.81) & 42.99 (46.01) \\
Clip-Higher
& 53.77 (\best{56.33}) & 73.49 (78.31) & 20.83 (24.17) & 43.33 (46.67) & 19.58 (21.67) & 26.67 (30.00) & 31.39 (33.36) & 47.83 (51.66) \\
Dual Clip
& 55.57 (55.72) & 77.11 (78.31) & 21.25 (24.17) & 36.67 (46.67) & 19.58 (22.50) & 26.67 (33.33) & 32.14 (33.34) & 46.81 (49.75) \\
Dynamic Clipping
& 53.16 (54.97) & 78.31 ({78.31}) & 23.75 (24.58) & \best{50.00} (\best{50.00}) & \best{20.83} (\best{22.92}) & \best{33.33} (\best{36.67}) & 32.58 (32.93) & \best{53.88} (\best{53.88}) \\
Clip-Cov
& 50.75 (55.12) & 69.88 (75.90) & 21.25 (22.50) & 46.67 (46.67) & \best{20.83} (22.08) & \best{33.33} (33.33) & 30.95 (32.48) & 49.96 (50.05) \\
Soft Gate
& 53.01 (53.01) & 74.70 (74.70) & 20.42 (22.92) & 43.33 (43.33) & 19.17 (19.17) & 26.67 (30.00) & 30.87 (31.17) & 48.23 (48.23) \\
ATR-based Clipping (ours)
& \best{56.02} (56.02) & \best{80.72} (\best{80.72}) & \best{25.42} (\best{25.83}) & \best{50.00} (\best{50.00}) & {19.58} (\best{22.92}) & {30.00} (\best{36.67}) & \best{33.67} (\best{33.67}) & 53.57 (53.57) \\
\bottomrule
\bottomrule
\end{tabular}
}
\end{table*}

Specifically, when the approximate trust-region constraint is unsatisfied ($X_{-}(s)$), ATR-based clipping maintains higher entropy than ratio-based clipping when there is an alignment between the advantage and policy likelihood (i.e., advantageous and high-probability actions).
By preserving entropy in these high-confidence regions, our ATR-based clipping method prioritizes stability, anchoring the policy to high-confidence actions, preventing the premature amplification of low-confidence tails that lead to instability.

Conversely, when the approximate trust-region constraint is satisfied ($X_{+}(s)$), our method increases entropy when there is a misalignment between the advantage and policy likelihood (i.e., advantageous yet low-probability actions).
This regime drives aggressive exploration, actively reallocating probability mass toward promising but underexplored parts of the action space, improving efficiency without compromising the approximate trust-region constraints.

\section{Experiments}\label{sec:experiment}

\subsection{Experiment Setting}
We implement all experiments using the Unsloth~\citep{unsloth} and TRL~\citep{vonwerra2022trl} frameworks. 
We fine-tune the Qwen3-1.7B and Qwen3-8B models~\citep{qwen3} on the DAPO-Math-17k dataset~\citep{yu2025dapo}, employing a sparse binary reward function, where the model receives a reward of $+1$ upon generating the correct final answer, and $0$ otherwise.
To mitigate computational overhead, we employ Low-Rank Adaptation (LoRA)~\citep{schulman2025lora}, which injects low-rank adapter matrices to optimize a minimal set of parameters, greatly reducing memory requirements without compromising performance~\citep{schulman2025lora}.

We compare our approach against SOTA clipping methods, including Clip-Higher~\citep{yu2025dapo}, Dynamic Clipping~\citep{yang2025dcpo}, Clip-Cov~\citep{cui2025entropy}, and Soft Gate~\citep{gao2025soft}. To ensure a fair comparison, we standardize Dr.GRPO~\citep{liu2025understanding} as the backbone RL algorithm across all baselines, varying only the clipping mechanism.
Consistent with this setup, we denote our method (ATR-GRPO) as ATR-based clipping in the experiments to explicitly highlight our primary distinction.
We train each method for 1,000 gradient steps, evaluating performance every 50 steps on the AMC2023~\citep{li2024numinamath}, AIME2024~\citep{aime_1983_2024}, and AIME2025~\citep{aime2025} benchmarks.
The implementation details are provided in~\Cref{appendix:implementation_detail}.
\begin{figure*}[t]
    \centering
    \includegraphics[width=0.5\linewidth]{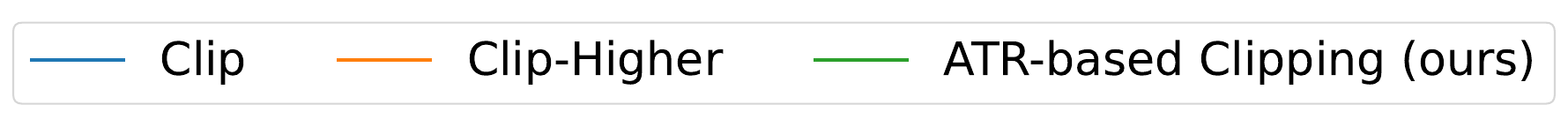}\\
    \subfigure[Return.\label{fig:reward}]{\includegraphics[width=0.24\linewidth]{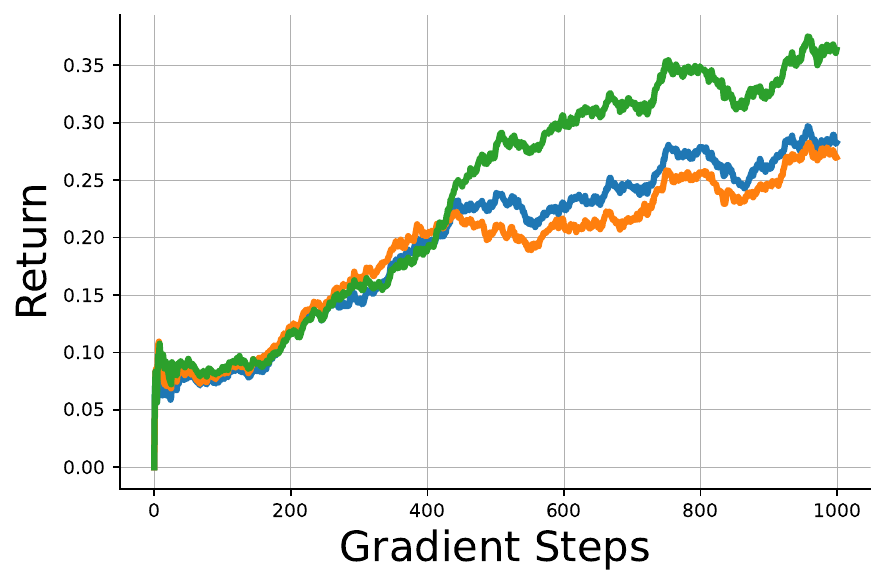}}
    \subfigure[Entropy.\label{fig:entropy}]{\includegraphics[width=0.24\linewidth]{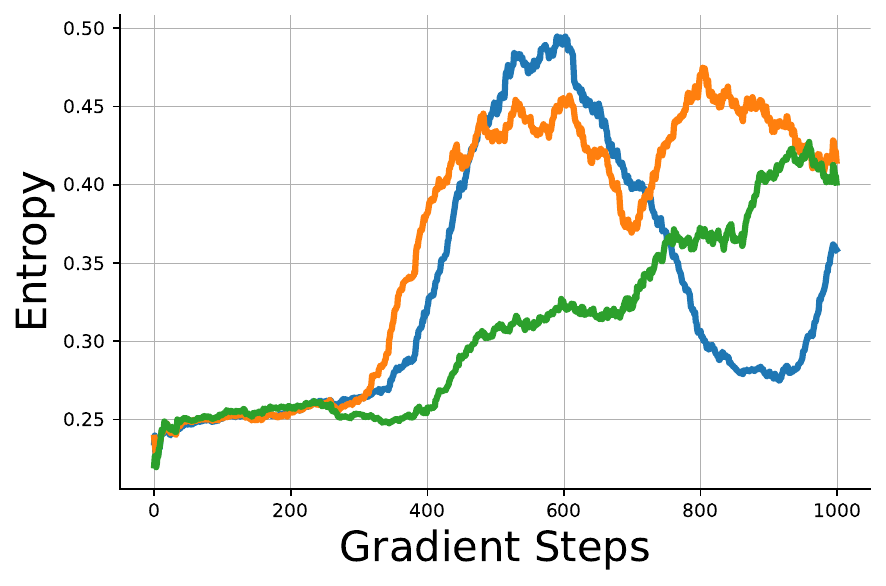}}
    \subfigure[Completion Length.\label{fig:completion_length}]{\includegraphics[width=0.24\linewidth]{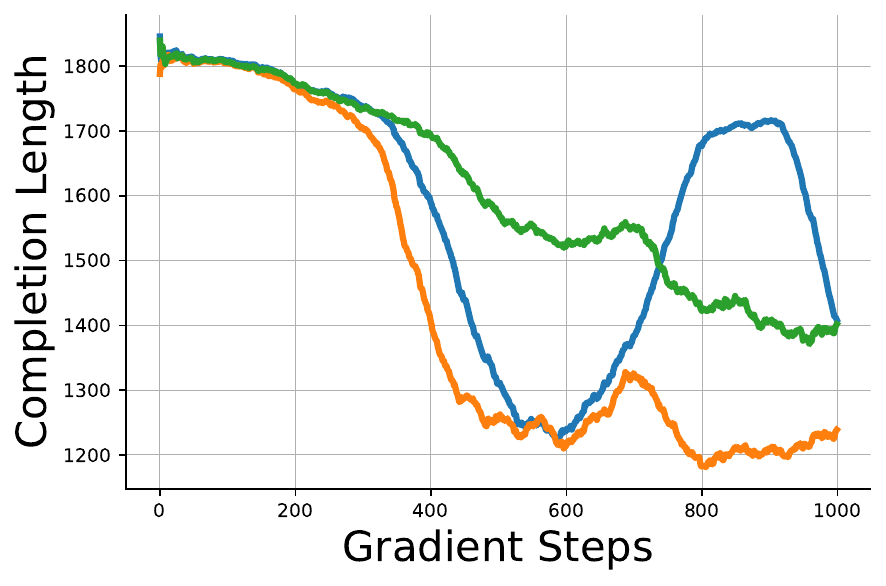}}
    \subfigure[Mean@8 (Average).\label{fig:reward_evaluation}]{\includegraphics[width=0.24\linewidth]{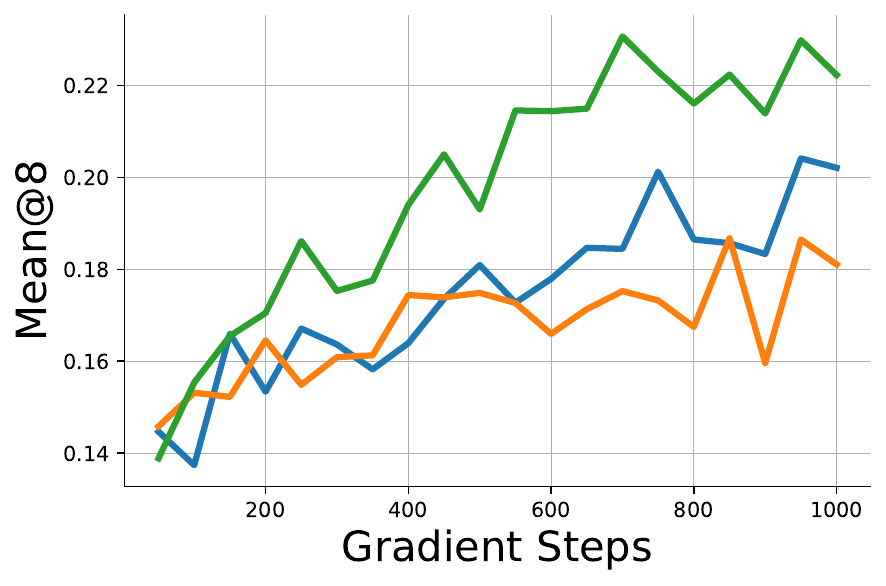}}
    \caption{Comparison of Clip, Clip-Higher, and ATR-based clipping on Qwen3-1.7B. The training curves for (a) return, (b) entropy, and (c) completion length are smoothed with a 100-step moving average window. (d) Evaluation performance of Mean@8 (Average).\label{fig:train_analysis}}
\end{figure*}

\subsection{Performance Comparison}
As shown in \Cref{table:qwen3_performance}, we conduct a comprehensive evaluation of the proposed ATR-based clipping method against SOTA baselines on the AMC2023, AIME2024, and AIME2025 benchmarks. 
We evaluate the performance of both Qwen3-1.7B and Qwen3-8B using Mean@8 and Pass@8. Mean@8 measures the average accuracy across 8 sampled responses, while Pass@8 indicates the success rate where at least one of the 8 samples is correct. 
We primarily report the evaluation performance of the \textbf{final} checkpoint, and additionally present the \textbf{best} evaluation performance achieved during training.

\paragraph{Performance on Qwen3-1.7B.}
Our proposed method achieves substantial improvements over the base model and consistently outperforms existing baselines on Qwen3-1.7B. Specifically, our method achieves the highest scores on the \textbf{final} evaluation performance (Average), with the Mean@8 of 22.93\% and Pass@8 of 42.18\%, surpassing the best baseline (Dual Clip). Specifically, for challenging AIME2025, our method achieves the \textbf{final} Mean@8 of 13.75\% compared to 13.33\% for Soft Gate, the best baseline.

\paragraph{Performance on Qwen3-8B.}
Our method continues to exhibit superior performance on Qwen3-8B, particularly on the AMC2023 and AIME2024. 
On AMC2023, our method achieves a \textbf{final} Mean@8 of 56.02\% and Pass@8 of 80.72\%.
Furthermore, on AIME2024, our approach demonstrates superior robustness, achieving a \textbf{final} Mean@8 of 25.42\% and Pass@8 of 50.00\%, outperforming Dynamic Clipping.
While Dynamic Clipping proves competitive on AIME2025, our method yields the highest \textbf{final} Mean@8 of 33.67\% (Average).

\subsection{Efficiency and Stability Analysis}
We now analyze the learning stability and efficiency of ATR-based clipping, comparing it primarily against the Clip and Clip Higher methods, the two baselines most closely related to and as computationally simple as ours. We present the training curves for return, entropy and completion length alongside evaluation performance (Mean@8) in~\Cref{fig:train_analysis}.

Specifically, as shown in~\Cref{fig:reward}, while the returns of all methods are initially similar, ATR-based clipping surpasses the baselines around $400$ gradient steps, continuing to improve to the return of $0.36$, whereas the baselines both plateau near $0.28$, demonstrating that our modification yields efficient performance improvements.
This trend is mirrored in the evaluation tasks (\Cref{fig:reward_evaluation}), where our method consistently outperforms baselines, confirming the superior sample efficiency of ATR-based clipping.

Moreover, our method does not sacrifice the learning stability. 
As presented in~\Cref{fig:entropy}, unlike the Clip method, ATR-based clipping maintains a steady, moderate entropy level, indicating a stable exploration strategy that avoids premature convergence or policy collapse.
In terms of the completion length (\Cref{fig:completion_length}), ATR-based clipping exhibits a stable, monotonic decrease. In contrast, the baselines show erratic oscillations (Clip) or sharp, potentially premature drops (Clip-Higher), showing our proposed method effectively stabilizes the learning process.

\subsection{Ablation Studies}

\paragraph{Trust Region Threshold $\delta$.}
Specifically, we show Mean@8 and Pass@8 with varying $\delta$ in~\Cref{fig:ablation_delta_mean_k} and~\Cref{fig:ablation_delta_pass_k}, respectively.
Both metrics exhibit a similar trend, steadily increasing from $\delta=0.05$ and  peaking at $\delta=0.07$.
However, further increasing $\delta$ results in sharp performance degeneration.
The results indicate that an overly small $\delta$ imposes stronger constraints that hinder performance improvement, while an excessively large $\delta$ leads to training instability or policy collapse, thereby degrading performance.
\begin{figure}[h]
    \centering
    \subfigure[Mean@8 with varying $\delta$.\label{fig:ablation_delta_mean_k}]{\includegraphics[width=0.45\linewidth]{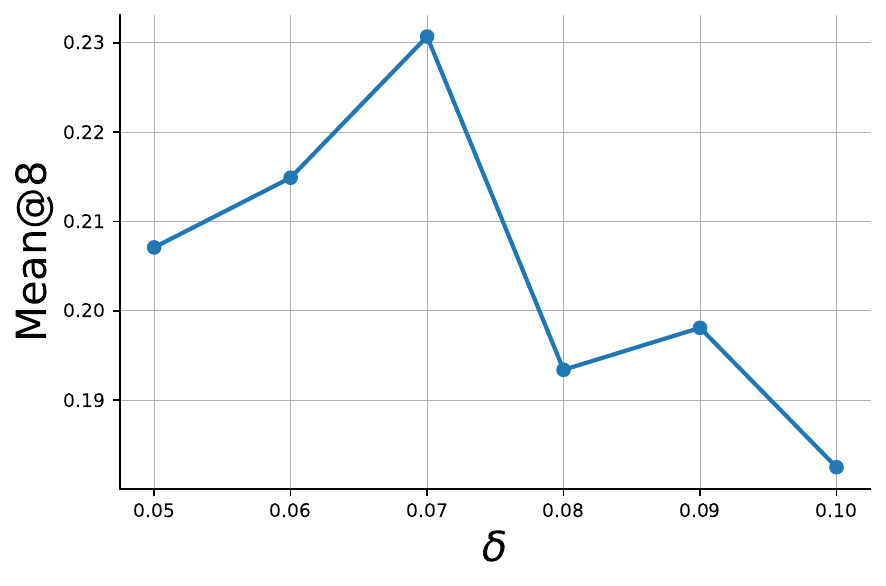}}
    \subfigure[Pass@8 with varying $\delta$.\label{fig:ablation_delta_pass_k}]{\includegraphics[width=0.45\linewidth]{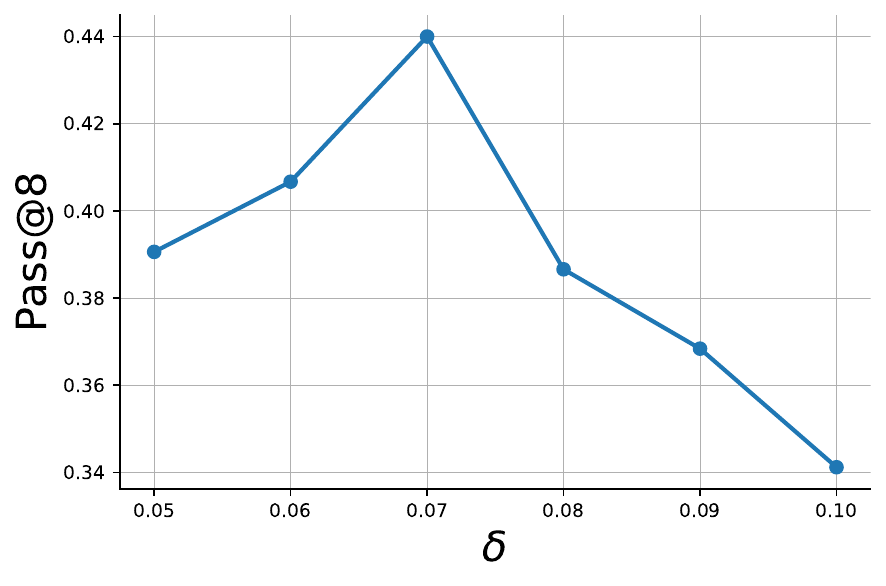}}
    \caption{Ablation experiments of ATR-GRPO on Qwen3-1.7B. The performance for (a) Mean@8 and (b) Pass@8 with varying $\delta$.\label{fig:ablation_analysis_delta}}
\end{figure}

\begin{table*}[t]
\centering
\caption{Performance comparison on Qwen3-1.7B across various KL estimators and fine-grained hyperparameter configurations for Clip and Clip-Higher. Each method is independently trained using 3 separate runs. We report the mean and standard deviation of the performance based on the final checkpoints from these 3 runs. The best performance is highlighted. \label{table:qwen3_performance_kl_estimator}}
\resizebox{0.97\textwidth}{!}{%
\begin{tabular}{lcccccccc}
\toprule
\toprule
\multicolumn{9}{c}{\textbf{\textit{Qwen3-1.7B}}} \\
\midrule
\multirow{2.5}{*}{\textbf{Method}} & \multicolumn{2}{c}{AMC2023} & \multicolumn{2}{c}{AIME2024} & \multicolumn{2}{c}{AIME2025} & \multicolumn{2}{c}{Average} \\
\cmidrule(lr){2-3} \cmidrule(lr){4-5} \cmidrule(lr){6-7} \cmidrule(lr){8-9}
 & Mean@8 (\%) & Pass@8 (\%) & Mean@8 (\%) & Pass@8 (\%) & Mean@8 (\%) & Pass@8 (\%) & Mean@8 (\%) & Pass@8 (\%) \\
\midrule
Clip ($\epsilon=0.2$) & $39.81_{\pm 0.26}$ & $61.45_{\pm 2.60}$ & $8.47_{\pm 0.86}$ & $26.67_{\pm 0.00}$ & $10.97_{\pm 0.39}$ & $22.22_{\pm 1.57}$ & $19.75_{\pm 0.10}$ & $36.78_{\pm 1.27}$ \\
Clip ($\epsilon \in \{0.1, 0.2, 0.3, 0.4, 0.5\}$) & $41.97_{\pm 0.14}$ & $65.86_{\pm 1.50}$ & $11.67_{\pm 1.02}$ & $27.78_{\pm 4.16}$ & \best{$14.17_{\pm 1.23}$} & \best{$28.89_{\pm 1.57}$} & $22.60_{\pm 0.63}$ & $40.84_{\pm 1.53}$ \\
Clip ($\epsilon=0.329$) & $40.16_{\pm 1.60}$ & $64.26_{\pm 3.98}$ & $11.81_{\pm 1.87}$ & $32.22_{\pm 1.57}$ & $11.25_{\pm 0.68}$ & $22.22_{\pm 1.57}$ & $21.07_{\pm 0.95}$ & $39.57_{\pm 2.33}$ \\
Clip ($\epsilon=0.422$) & $40.16_{\pm 0.14}$ & $63.05_{\pm 2.27}$ & $8.61_{\pm 1.19}$ & $24.44_{\pm 3.14}$ & $10.14_{\pm 0.39}$ & $23.33_{\pm 0.00}$ & $19.64_{\pm 0.38}$ & $36.94_{\pm 0.29}$ \\
Clip-Higher ($\epsilon_l=0.2, \epsilon_u=0.28$) & $36.75_{\pm 0.56}$ & $64.66_{\pm 3.16}$ & $9.17_{\pm 0.59}$ & $21.11_{\pm 1.57}$ & $8.47_{\pm 0.86}$ & $20.00_{\pm 2.72}$ & $18.13_{\pm 0.64}$ & $35.26_{\pm 2.10}$ \\
Clip-Higher ($\epsilon_l=0.329, \epsilon_u=0.5$) & $39.66_{\pm 0.61}$ & $61.85_{\pm 2.05}$ & $7.64_{\pm 0.39}$ & $22.22_{\pm 3.14}$ & $11.81_{\pm 0.20}$ & $23.33_{\pm 2.72}$ & $19.70_{\pm 0.24}$ & $35.80_{\pm 2.31}$ \\
Clip-Higher ($\epsilon_l=0.329, \epsilon_u=0.6$) & $42.72_{\pm 0.68}$ & $66.67_{\pm 1.50}$ & $10.83_{\pm 0.34}$ & \best{$27.78_{\pm 1.57}$} & $12.22_{\pm 0.39}$ & $23.33_{\pm 2.72}$ & $21.93_{\pm 0.21}$ & $39.26_{\pm 0.60}$ \\
\midrule
$\text{KL}1 (\delta=0.07)$ & $41.47_{\pm 1.40}$ & $65.86_{\pm 2.27}$ & $8.19_{\pm 0.20}$ & $24.44_{\pm 3.14}$ & $12.78_{\pm 1.19}$ & $24.44_{\pm 5.67}$ & $20.81_{\pm 0.77}$ & $38.25_{\pm 3.67}$ \\
$\text{KL}2 (\delta=0.07)$ & \best{$43.22_{\pm 1.23}$} & $67.07_{\pm 1.14}$ & $11.11_{\pm 1.37}$ & $26.67_{\pm 2.72}$ & $13.33_{\pm 0.90}$ & $24.44_{\pm 3.14}$ & $22.56_{\pm 0.66}$ & $39.39_{\pm 1.75}$ \\
Full KL-Guided Clipping $(\delta=0.07)$~\citep{wang2019trust} & $38.70_{\pm 0.44}$ & $62.65_{\pm 0.98}$ & $8.47_{\pm 1.37}$ & $23.33_{\pm 2.72}$ & $10.14_{\pm 0.86}$ & $21.11_{\pm 1.57}$ & $19.11_{\pm 0.65}$ & $35.70_{\pm 0.29}$ \\
IS-weighted \KLThreeText ($\delta=0.07$)& $39.26_{\pm 0.40}$ & $64.66_{\pm 1.50}$ & $8.75_{\pm 1.48}$ & $26.67_{\pm 4.71}$ & $11.67_{\pm 0.68}$ & $24.44_{\pm 1.57}$ & $19.89_{\pm 0.38}$ & $38.59_{\pm 1.98}$ \\
\midrule
\KLThreeText ($\delta=0.07, \lowerClippingRange[\delta][\KLThree[]]=0.671, \upperClippingRange[\delta][\KLThree[]]=1.422$) & {$43.17_{\pm 1.17}$} & \best{$68.67_{\pm 0.98}$} & \best{$13.06_{\pm 1.71}$} & \best{$27.78_{\pm 4.16}$} & {$14.03_{\pm 1.71}$} & \best{$28.89_{\pm 1.57}$} & \best{$23.42_{\pm 1.04}$} & \best{$41.78_{\pm 2.00}$} \\
\bottomrule
\bottomrule
\end{tabular}
}
\end{table*}
\paragraph{Fine-grained Tuning Clip and Clip-Higher.}
To ensure that our performance gains stem from fundamental algorithmic improvements rather than sub-optimal baseline configurations, we conducted extensive hyperparameter tuning on \textbf{Clip} and \textbf{Clip-Higher} baselines.
Specifically, we evaluated \textbf{Clip} with $\epsilon \in \{0.1, 0.2, 0.3, 0.4, 0.5\}$, reporting the performance achieved by the optimal configuration.
%
%
Additionally, for the trust region threshold $\delta=0.07$, we have corresponding ATR-based clipping range of $[\lowerClippingRange[\delta][\KLThree[]], \upperClippingRange[\delta][\KLThree[]]]=[0.671, 1.422]$.
We specifically configured Clip and Clip-Higher using these two ranges, including \textbf{Clip} ($\epsilon\in \{0.329, 0.422\}$) and \textbf{Clip-Higher} ($\epsilon_l=0.329$, $\epsilon_u\in \{0.5, 0.6\}$).
\textbf{As shown in \Cref{table:qwen3_performance_kl_estimator}, our method achieves the highest performance of Mean@8 and Pass@8 (Average) compared to these meticulously tuned baselines. Crucially, our method outperforms \textbf{Clip} ($\epsilon = 0.329$ and $\epsilon = 0.422$), validating our previous statement and theoretical analysis (\Cref{theorem:entropy_diff}).}

\paragraph{KL Estimators.}
Furthermore, we evaluate alternative KL estimators (KL1, KL2, Full KL-guided clipping, and IS-weighted KL3) under the same trust region threshold ($\delta=0.07$), as presented in~\Cref{table:qwen3_performance_kl_estimator}.
The results demonstrate that \KLThreeText consistently yields the performance gains.
While recent studies~\Citep{shah2025comedy, liu2025rethinking, zhang2025design, liu2025deepseek} have explored KL estimators primarily as loss regularization terms or auxiliary reward signals, we investigate a different application, utilizing the KL estimator to guide the clipping criterion.
We empirically demonstrate that \KLThreeText is the best choice among these KL estimators for this specific purpose.

\paragraph{Test-time Sample Budget.}
Then, we further report Mean@K (\Cref{fig:ablation_mean_varying_k}) and Pass@K (\Cref{fig:ablation_pass_varying_k}) for ATR-GRPO on Qwen3-1.7B with varying K. 
As K increases, Mean@K remains stable around $0.23$.
In contrast, Pass@K exhibits an increasing trend, scaling with the test-time sample budget from 24.36\% at Pass@1 to 56.29\% at Pass@128.
\subsection{Limitations and Future Works}

\paragraph{Adaptive Trust Region Threshold.}
Employing a static trust region threshold throughout the entire training process may be suboptimal.
Inspired by existing works~\citep{wang2019trust, yang2025dcpo}, we will investigate mechanisms to adaptively adjust $\delta$ based on the policy probability and entropy for more stable and efficient learning in the future.

\paragraph{Sequence-level Integration.}
While this work primarily focuses on the token-level objective, we recognize the limitation posed by the mismatch between token-level importance sampling and sequence-level rewards, leading to high-variance gradients and unstable training~\citep{zheng2025group}. We aim to explore the extension of our framework to sequence-level objectives in future work.

\begin{figure}[t]
    \centering
    \subfigure[Mean@K with varying K.\label{fig:ablation_mean_varying_k}]{\includegraphics[width=0.45\linewidth]{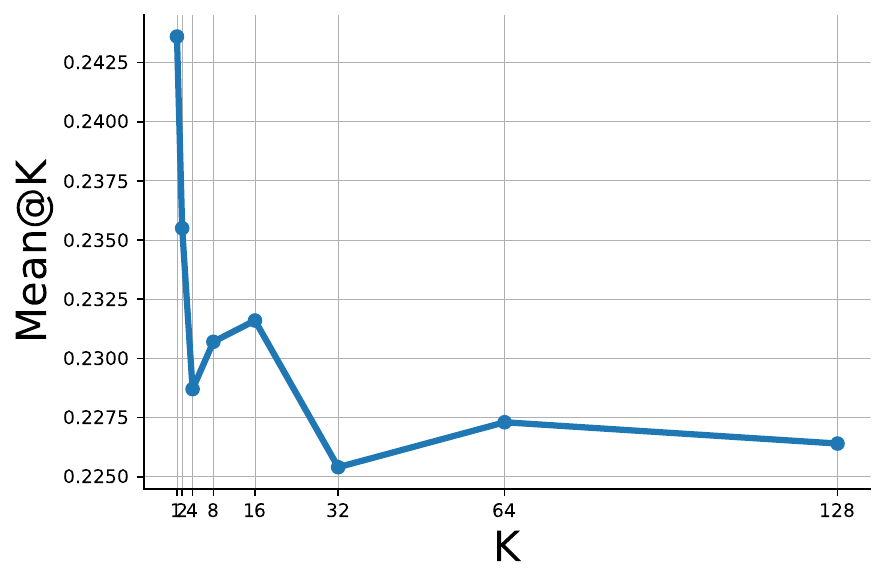}}
    \subfigure[Pass@K with varying K.\label{fig:ablation_pass_varying_k}]{\includegraphics[width=0.45\linewidth]{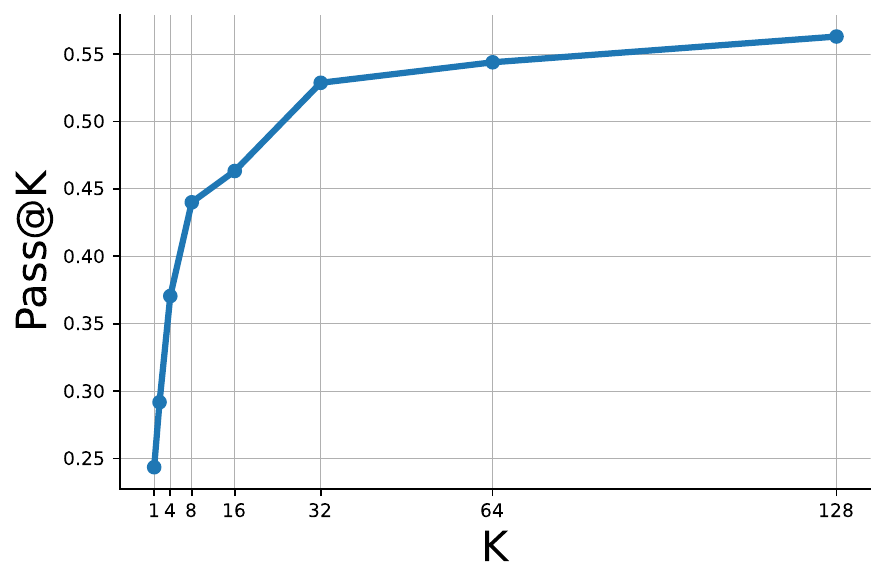}}
    \caption{The performance (Average) of ATR-GRPO on Qwen3-1.7B for (a) Mean@K and (b) Pass@K with varying K.\label{fig:ablation_analysis_k}}
\end{figure}
\section{Conclusion}\label{sec:conclusion}
In this work, we introduce a unified clipping framework for policy optimization that generalizes the notion of policy divergence, encompassing both ratio–based and KL-based constraints.
From this foundational perspective, we analyze and identify the \KLThreeText estimator as a particularly effective policy divergence measure. 
Based on this insight, we proposed ATR-GRPO that actively reallocates probability mass toward promising actions rather than passively truncating updates through ratio-based clipping. 
Our theoretical analysis demonstrates that ATR-GRPO enables more principled and effective exploration than standard symmetric ratio-based clipping, while preserving the computational efficiency of GRPO.
Empirically, comprehensive experiments across multiple verifiable mathematical reasoning benchmarks demonstrate that ATR-GRPO yields performance improvements in both training stability and final performance compared to existing baselines. 
These results underscore the central role of the policy divergence constraint in policy optimization and suggest that exploring alternative policy divergence measures is a promising direction for advancing LLMs.

\clearpage

\bibliography{example_paper}
\bibliographystyle{icml2026}

\newpage
\appendix
\crefalias{section}{appendix}
\crefalias{subsection}{appendix}
\onecolumn

\section{Proof Details}

\subsection{Equivalence and Asymmetry Analysis}\label{appendix:theorem_KL3}
\begin{theorem}[Equivalence and Asymmetry]
Define the lower and upper clipping ranges for a given threshold $\delta > 0$ as:
\[
\lowerClippingRange[\delta][\KLThree[]] = \min_{ \theta } \ratio(\theta) \quad \text{s.t.} \quad \KLThree(\theta) \leq \delta,
\]
\[
\upperClippingRange[\delta][\KLThree[]] = \max_{ \theta } \ratio(\theta) \quad \text{s.t.} \quad \KLThree(\theta) \leq \delta.
\]
(1) The constraint $\KLThree(\theta) \leq \delta$ is equivalent to $\lowerClippingRange[\delta][\KLThree[]] \leq \ratio(\theta) \leq \upperClippingRange[\delta][\KLThree[]]$, where $0 < \lowerClippingRange[\delta][\KLThree[]] < 1 < \upperClippingRange[\delta][\KLThree[]]$.

(2) These ranges satisfy the asymmetry property: $1 - \lowerClippingRange[\delta][\KLThree[]] < \upperClippingRange[\delta][\KLThree[]] - 1$.
\end{theorem}
\begin{proof}
For a given trust region threshold $\delta > 0$, we aim to find the area $(0<\lowerClippingRange[\delta][\KLThree[]] < 1 < \upperClippingRange[\delta][\KLThree[]])$ such that $\KLThree(\theta) \leq \delta, \forall \ratio(\theta) \in [\lowerClippingRange[\delta][\KLThree[]], \upperClippingRange[\delta][\KLThree[]]]$.

Recall that
\begin{equation}
    \KLThree(\theta) \coloneqq \ratio(\theta) - 1 - \log \ratio(\theta)
\end{equation}

We have
\begin{equation}
    \begin{aligned}
        \ratio(\theta) - 1 - \log \ratio(\theta) = \delta & \\
        \log \ratio(\theta) = \ratio(\theta) - 1 - \delta & \\
        \ratio(\theta) = e^{\ratio(\theta) - 1 - \delta} & \\
        \ratio(\theta)e^{-\ratio(\theta)} = e^{- 1 - \delta} & \\
        -\ratio(\theta)e^{-\ratio(\theta)} = -e^{- 1 - \delta} & \\
    \end{aligned}
\end{equation}

The Lambert $W$ function is defined as $W(z) e^{W(z)} = z$. 
Then, let $\ratio(\theta) = -W(z)$ and $z=-e^{- 1 - \delta}$, we have
\begin{equation}
    \begin{aligned}
        & \ratio(\theta) = -W(-e^{- 1 - \delta}) \\
    \end{aligned}
\end{equation}

The equation has two real-valued Lambert solutions corresponding to the two real branches $W_{0}$ and $W_{-1}$ for $\lowerClippingRange[\delta][\KLThree[]]$ and $\upperClippingRange[\delta][\KLThree[]]$, respectively.
Specifically, we have
\begin{equation}
    \begin{aligned}
        & \lowerClippingRange[\delta][\KLThree[]] = -W_{0}(-e^{- 1 - \delta}) \in (0, 1)\\
        & \upperClippingRange[\delta][\KLThree[]] = -W_{-1}(-e^{- 1 - \delta}) \in (1, +\inf)\\
    \end{aligned}
\end{equation}

Then, we show that we always have $1 - \lowerClippingRange[\delta][\KLThree[]] < \upperClippingRange[\delta][\KLThree[]] - 1$.

We have
\begin{equation}
\begin{split}
    \frac{d}{d\ratio(\theta)}\KLThree(\theta)=&1-\frac{1}{\ratio(\theta)},\\ 
    \frac{d^2}{d(\ratio(\theta))^2}\KLThree(\theta)=&\frac{1}{(\ratio(\theta))^2}>0.\\
\end{split}
\end{equation}

Specifically, we have
\begin{equation}
\left|\frac{d}{d\ratio(\theta)}\KLThree(\theta)\right| = 
\begin{cases}
\frac{1}{\ratio(\theta)}-1, & \ratio(\theta)<1,\\
1-\frac{1}{\ratio(\theta)}, & \ratio(\theta)>1,
\end{cases}
\end{equation}

We can observe that as $\ratio(\theta) \to 0^+$, the slope magnitude 
$\left|\frac{d}{d\ratio(\theta)}\KLThree(\theta)\right|$ diverges to infinity,
whereas for $\ratio(\theta)>1$ it is uniformly bounded by $1$.
This implies that $\text{kl}_3(\ratio(\theta))$ increases more rapidly for
$\ratio(\theta)<1$ than for $\ratio(\theta)>1$.

Recall that we have $\lowerClippingRange[\delta][\KLThree[]]<1<\upperClippingRange[\delta][\KLThree[]]$.
Since $\KLThree(\theta)$ is strictly convex on $(0,\infty)$ and satisfies
$\left|\KLThree'(\theta)\right| \to \infty$ as $\ratio(\theta)\to0^+$ while
$\left|\KLThree'(\theta)\right|<1$ for $\ratio(\theta)>1$,
the function grows faster for $\ratio(\theta)<1$ than for $\ratio(\theta)>1$.
Therefore, to reach the same divergence level $\delta$, the solution
above $1$ must deviate farther from $1$ than the solution below $1$,
which implies that we always have
\begin{equation}
1-\lowerClippingRange[\delta][\KLThree[]]
<
\upperClippingRange[\delta][\KLThree[]]-1.
\end{equation}

\end{proof}

\subsection{Policy Logits Difference Analysis}
\label{appendix:proof:performance_diff}

\begin{theorem}[Policy Logits Difference]
Consider the policy gradient algorithm with learning rate $\eta$ and given the state visitation distribution $d^{\pi_{\theta_\text{old}}(s)}$ induced by $\pi_{\theta_\text{old}}$. 
Let $\Delta \theta_{s, a}$ denote the policy logits difference between the ATR-based and ratio-based clipping methods. 

Under event $X_{-}(s)$, we have 
$
\begin{aligned}
    &\Delta \theta_{s, a} = - \eta d^{\pi_{\theta_\text{old}}(s)}
        \pi_{\theta^k}(a|s)
        \left[
            A\mathbb{I}_{X_{-}(s)}(a) 
            - 
            \mathop{\mathbb{E}}_{a' \sim \pi_{\theta^k}(\cdot|s)}
            [A\mathbb{I}_{X_{-}(s)}(a')]
        \right].\\
\end{aligned}
$

Under event $X_{+}(s)$, we have
$
\begin{aligned}
    &\Delta \theta_{s, a} = \eta d^{\pi_{\theta_\text{old}}(s)}
        \pi_{\theta^k}(a|s)
        \left[
            A\mathbb{I}_{X_{+}(s)}(a) 
            - 
            \mathop{\mathbb{E}}_{a' \sim \pi_{\theta^k}(\cdot|s)}
            [A\mathbb{I}_{X_{+}(s)}(a')]
        \right].\\
\end{aligned}
$
\end{theorem}

\begin{proof}
Consider the clipped surrogate objective of ratio-based clipping:
\begin{equation}
    \mathcal{J}^\text{ratio}(\theta) := \mathop{\mathbb{E}}_{\substack{y \sim \pi_{\theta_{\text{old}}}(\cdot|x) \\ x \sim D}} \bigg[
    \frac{1}{|y|}\sum_{t=1}^{|y|} \mathrm{clip}^{\text{ratio}}\left(w_{t}(\theta), 1-\epsilon, 1+\epsilon\right)A_t \bigg].
\end{equation}

Consider the clipped surrogate objective of ATR-based clipping:
\begin{equation}
    \mathcal{J}^\text{ATR}(\theta) := \mathop{\mathbb{E}}_{\substack{y \sim \pi_{\theta_{\text{old}}}(\cdot|x) \\ x \sim D}} \bigg[
    \frac{1}{|y|}\sum_{t=1}^{|y|} \mathrm{clip}^{\text{ratio}}\left(w_{t}(\theta), \lowerClippingRange[\delta][\KLThree[]], \upperClippingRange[\delta][\KLThree[]]\right)A_t \bigg].
\end{equation}

Following the similar proof sketch of~\citep{park2025clip}, for a given trust region threshold $\delta$ and the corresponding ATR-based clipping range $(\lowerClippingRange[\delta][\KLThree[]], \upperClippingRange[\delta][\KLThree[]])$, we have
\begin{equation}
    \begin{aligned}
        &\frac{1}{d^{\pi_{\theta_\text{old}}(s)}}\frac{\partial}{\partial\theta_{s,a}}\mathcal{J}^\text{ATR}(\theta)\\
        =&\mathbb{P}(A > 0, \lowerClippingRange[\delta][\KLThree[]] \leq \ratio(\theta_{s,a'}) \leq \upperClippingRange[\delta][\KLThree[]]) \mathop{\mathbb{E}}_{a' \sim \pi_{\theta_\text{old}}(\cdot|s)}[\frac{\partial}{\partial\theta_{s,a}}\ratio(\theta_{s,a'})A| A > 0, \lowerClippingRange[\delta][\KLThree[]] \leq \ratio(\theta_{s,a'}) \leq \upperClippingRange[\delta][\KLThree[]]] \\
        & + \mathbb{P}(A > 0, \ratio(\theta_{s,a'}) \leq \lowerClippingRange[\delta][\KLThree[]]) \mathop{\mathbb{E}}_{a' \sim \pi_{\theta_\text{old}}(\cdot|s)}[\frac{\partial}{\partial\theta_{s,a}}\lowerClippingRange[\delta][\KLThree[]]A| A > 0, \ratio(\theta_{s,a'}) \leq \lowerClippingRange[\delta][\KLThree[]]] \\
        & + \mathbb{P}(A > 0, \ratio(\theta_{s,a'}) \geq \upperClippingRange[\delta][\KLThree[]]) \mathop{\mathbb{E}}_{a' \sim \pi_{\theta_\text{old}}(\cdot|s)}[\frac{\partial}{\partial\theta_{s,a}}\upperClippingRange[\delta][\KLThree[]]A| A > 0, \ratio(\theta_{s,a'}) \geq \upperClippingRange[\delta][\KLThree[]]] \\
        & + \mathbb{P}(A < 0, \lowerClippingRange[\delta][\KLThree[]] \leq \ratio(\theta_{s,a'}) \leq \upperClippingRange[\delta][\KLThree[]]) \mathop{\mathbb{E}}_{a' \sim \pi_{\theta_\text{old}}(\cdot|s)}[\frac{\partial}{\partial\theta_{s,a}}\ratio(\theta_{s,a'})A| A < 0, \lowerClippingRange[\delta][\KLThree[]] \leq \ratio(\theta_{s,a'}) \leq \upperClippingRange[\delta][\KLThree[]]] \\
        & + \mathbb{P}(A < 0, \ratio(\theta_{s,a'}) \leq \lowerClippingRange[\delta][\KLThree[]]) \mathop{\mathbb{E}}_{a' \sim \pi_{\theta_\text{old}}(\cdot|s)}[\frac{\partial}{\partial\theta_{s,a}}\lowerClippingRange[\delta][\KLThree[]]A| A < 0, \ratio(\theta_{s,a'}) \leq \lowerClippingRange[\delta][\KLThree[]]] \\
        & + \mathbb{P}(A < 0, \ratio(\theta_{s,a'}) \geq \upperClippingRange[\delta][\KLThree[]]) \mathop{\mathbb{E}}_{a' \sim \pi_{\theta_\text{old}}(\cdot|s)}[\frac{\partial}{\partial\theta_{s,a}}\upperClippingRange[\delta][\KLThree[]]A| A < 0, \ratio(\theta_{s,a'}) \geq \upperClippingRange[\delta][\KLThree[]]] \\
        =& \mathbb{P}(A > 0, \lowerClippingRange[\delta][\KLThree[]] \leq \ratio(\theta_{s,a'}) \leq \upperClippingRange[\delta][\KLThree[]]) \mathop{\mathbb{E}}_{a' \sim \pi_{\theta_\text{old}}(\cdot|s)}[\frac{\partial}{\partial\theta_{s,a}}\ratio(\theta_{s,a'})A| A > 0, \lowerClippingRange[\delta][\KLThree[]] \leq \ratio(\theta_{s,a'}) \leq \upperClippingRange[\delta][\KLThree[]]] \\
        & + \mathbb{P}(A < 0, \lowerClippingRange[\delta][\KLThree[]] \leq \ratio(\theta_{s,a'}) \leq \upperClippingRange[\delta][\KLThree[]]) \mathop{\mathbb{E}}_{a' \sim \pi_{\theta_\text{old}}(\cdot|s)}[\frac{\partial}{\partial\theta_{s,a}}\ratio(\theta_{s,a'})A| A < 0, \lowerClippingRange[\delta][\KLThree[]] \leq \ratio(\theta_{s,a'}) \leq \upperClippingRange[\delta][\KLThree[]]]. \\
    \end{aligned}
\end{equation}

Similarly, for the ratio-based clipping $(1 - \epsilon, 1 + \epsilon)$, we have
\begin{equation}
    \begin{aligned}
        &\frac{1}{d^{\pi_{\theta_\text{old}}(s)}}\frac{\partial}{\partial\theta_{s,a}}\mathcal{J}^\text{ratio}(\theta)\\
        =
        &\mathbb{P}(A > 0, 1 - \epsilon \leq \ratio(\theta_{s,a'}) \leq 1 + \epsilon) \mathop{\mathbb{E}}_{a' \sim \pi_{\theta_\text{old}}(\cdot|s)}[\frac{\partial}{\partial\theta_{s,a}}\ratio(\theta_{s,a'})A| A > 0, 1 - \epsilon \leq \ratio(\theta_{s,a'}) \leq 1 + \epsilon] \\
        & + \mathbb{P}(A < 0, 1 - \epsilon \leq \ratio(\theta_{s,a'}) \leq 1 + \epsilon) \mathop{\mathbb{E}}_{a' \sim \pi_{\theta_\text{old}}(\cdot|s)}[\frac{\partial}{\partial\theta_{s,a}}\ratio(\theta_{s,a'})A| A < 0, 1 - \epsilon \leq \ratio(\theta_{s,a'}) \leq 1 + \epsilon]. \\
    \end{aligned}
\end{equation}

Since we apply the policy gradient algorithm, and we have that
\begin{equation}
    \begin{aligned}
        \theta^{\text{ATR}, k+1}_{s, a} - \theta^{k}_{s, a} =& \eta \frac{\partial}{\partial\theta_{s,a}}\mathcal{J}^\text{ATR}(\theta) \\
        \theta^{\text{ratio,}k+1}_{s, a} - \theta^{k}_{s, a} =& \eta \frac{\partial}{\partial\theta_{s,a}}\mathcal{J}^\text{ratio}(\theta) \\
    \end{aligned}
\end{equation}

Then, we have
\begin{equation}
    \begin{aligned}
        &\frac{1}{\eta d^{\pi_{\theta_\text{old}}(s)}} (\theta^{\text{ATR}, k+1}_{s, a} - \theta^{\text{ratio,}k+1}_{s, a})\\
        = & \mathbb{P}(A > 0, \lowerClippingRange[\delta][\KLThree[]] \leq \ratio(\theta_{s,a'}) \leq \upperClippingRange[\delta][\KLThree[]]) \mathop{\mathbb{E}}_{a' \sim \pi_{\theta_\text{old}}(\cdot|s)}[\frac{\partial}{\partial\theta_{s,a}}\ratio(\theta_{s,a'})A| A > 0, \lowerClippingRange[\delta][\KLThree[]] \leq \ratio(\theta_{s,a'}) \leq \upperClippingRange[\delta][\KLThree[]]]\\
        & - \mathbb{P}(A > 0, 1 - \epsilon \leq \ratio(\theta_{s,a'}) \leq 1 + \epsilon) \mathop{\mathbb{E}}_{a' \sim \pi_{\theta_\text{old}}(\cdot|s)}[\frac{\partial}{\partial\theta_{s,a}}\ratio(\theta_{s,a'})A| A > 0, 1 - \epsilon \leq \ratio(\theta_{s,a'}) \leq 1 + \epsilon] \\
        & + \mathbb{P}(A < 0, \lowerClippingRange[\delta][\KLThree[]] \leq \ratio(\theta_{s,a'}) \leq \upperClippingRange[\delta][\KLThree[]]) \mathop{\mathbb{E}}_{a' \sim \pi_{\theta_\text{old}}(\cdot|s)}[\frac{\partial}{\partial\theta_{s,a}}\ratio(\theta_{s,a'})A| A < 0, \lowerClippingRange[\delta][\KLThree[]] \leq \ratio(\theta_{s,a'}) \leq \upperClippingRange[\delta][\KLThree[]]] \\
        & - \mathbb{P}(A < 0, 1 - \epsilon \leq \ratio(\theta_{s,a'}) \leq 1 + \epsilon) \mathop{\mathbb{E}}_{a' \sim \pi_{\theta_\text{old}}(\cdot|s)}[\frac{\partial}{\partial\theta_{s,a}}\ratio(\theta_{s,a'})A| A < 0, 1 - \epsilon \leq \ratio(\theta_{s,a'}) \leq 1 + \epsilon] \\
    \end{aligned}
\end{equation}

Now, we consider the first case: $\upperClippingRange[\delta][\KLThree[]] = 1 + \epsilon$ and $1 - \epsilon < \lowerClippingRange[\delta][\KLThree[]]$.
In this case, we have
\begin{equation}
    \begin{aligned}
        & \mathbb{P}(A > 0, \lowerClippingRange[\delta][\KLThree[]] \leq \ratio(\theta_{s,a'}) \leq \upperClippingRange[\delta][\KLThree[]]) \mathop{\mathbb{E}}_{a' \sim \pi_{\theta_\text{old}}(\cdot|s)}[\frac{\partial}{\partial\theta_{s,a}}\ratio(\theta_{s,a'})A| A > 0, \lowerClippingRange[\delta][\KLThree[]] \leq \ratio(\theta_{s,a'}) \leq \upperClippingRange[\delta][\KLThree[]]]\\
        & - \mathbb{P}(A > 0, 1 - \epsilon \leq \ratio(\theta_{s,a'}) \leq \upperClippingRange[\delta][\KLThree[]]) \mathop{\mathbb{E}}_{a' \sim \pi_{\theta_\text{old}}(\cdot|s)}[\frac{\partial}{\partial\theta_{s,a}}\ratio(\theta_{s,a'})A| A > 0, 1 - \epsilon \leq \ratio(\theta_{s,a'}) \leq \upperClippingRange[\delta][\KLThree[]]] \\
        = & \mathbb{P}(A > 0, \lowerClippingRange[\delta][\KLThree[]] \leq \ratio(\theta_{s,a'}) \leq \upperClippingRange[\delta][\KLThree[]]) \mathop{\mathbb{E}}_{a' \sim \pi_{\theta_\text{old}}(\cdot|s)}[\frac{\partial}{\partial\theta_{s,a}}\ratio(\theta_{s,a'})A| A > 0, \lowerClippingRange[\delta][\KLThree[]] \leq \ratio(\theta_{s,a'}) \leq \upperClippingRange[\delta][\KLThree[]]]\\
        & - \mathbb{P}(A > 0, 1 - \epsilon \leq \ratio(\theta_{s,a'}) \leq \lowerClippingRange[\delta][\KLThree[]]) \mathop{\mathbb{E}}_{a' \sim \pi_{\theta_\text{old}}(\cdot|s)}[\frac{\partial}{\partial\theta_{s,a}}\ratio(\theta_{s,a'})A| A > 0, 1 - \epsilon \leq \ratio(\theta_{s,a'}) \leq \lowerClippingRange[\delta][\KLThree[]]] \\
        & - \mathbb{P}(A > 0, \lowerClippingRange[\delta][\KLThree[]] \leq \ratio(\theta_{s,a'}) \leq \upperClippingRange[\delta][\KLThree[]]) \mathop{\mathbb{E}}_{a' \sim \pi_{\theta_\text{old}}(\cdot|s)}[\frac{\partial}{\partial\theta_{s,a}}\ratio(\theta_{s,a'})A| A > 0, \lowerClippingRange[\delta][\KLThree[]] \leq \ratio(\theta_{s,a'}) \leq \upperClippingRange[\delta][\KLThree[]]] \\
        = & - \mathbb{P}(A > 0, 1 - \epsilon \leq \ratio(\theta_{s,a'}) \leq \lowerClippingRange[\delta][\KLThree[]]) \mathop{\mathbb{E}}_{a' \sim \pi_{\theta_\text{old}}(\cdot|s)}[\frac{\partial}{\partial\theta_{s,a}}\ratio(\theta_{s,a'})A| A > 0, 1 - \epsilon \leq \ratio(\theta_{s,a'}) \leq \lowerClippingRange[\delta][\KLThree[]]] \\
    \end{aligned}
\end{equation}

Similarly, we have
\begin{equation}
    \begin{aligned}
        & \mathbb{P}(A < 0, \lowerClippingRange[\delta][\KLThree[]] \leq \ratio(\theta_{s,a'}) \leq \upperClippingRange[\delta][\KLThree[]]) \mathop{\mathbb{E}}_{a' \sim \pi_{\theta_\text{old}}(\cdot|s)}[\frac{\partial}{\partial\theta_{s,a}}\ratio(\theta_{s,a'})A| A < 0, \lowerClippingRange[\delta][\KLThree[]] \leq \ratio(\theta_{s,a'}) \leq \upperClippingRange[\delta][\KLThree[]]]\\
        & - \mathbb{P}(A < 0, 1 - \epsilon \leq \ratio(\theta_{s,a'}) \leq \upperClippingRange[\delta][\KLThree[]]) \mathop{\mathbb{E}}_{a' \sim \pi_{\theta_\text{old}}(\cdot|s)}[\frac{\partial}{\partial\theta_{s,a}}\ratio(\theta_{s,a'})A| A < 0, 1 - \epsilon \leq \ratio(\theta_{s,a'}) \leq \upperClippingRange[\delta][\KLThree[]]] \\
        = & - \mathbb{P}(A < 0, 1 - \epsilon \leq \ratio(\theta_{s,a'}) \leq \lowerClippingRange[\delta][\KLThree[]]) \mathop{\mathbb{E}}_{a' \sim \pi_{\theta_\text{old}}(\cdot|s)}[\frac{\partial}{\partial\theta_{s,a}}\ratio(\theta_{s,a'})A| A < 0, 1 - \epsilon \leq \ratio(\theta_{s,a'}) \leq \lowerClippingRange[\delta][\KLThree[]]] \\
    \end{aligned}
\end{equation}

Therefore, we have

\begin{equation}
    \begin{aligned}
        & \frac{1}{\eta d^{\pi_{\theta_\text{old}}(s)}} (\theta^{\text{ATR}, k+1}_{s, a} - \theta^{\text{ratio,}k+1}_{s, a}) \\
        = & - \mathbb{P}(A > 0, 1 - \epsilon \leq \ratio(\theta_{s,a'}) \leq \lowerClippingRange[\delta][\KLThree[]]) \mathop{\mathbb{E}}_{a' \sim \pi_{\theta_\text{old}}(\cdot|s)}[\frac{\partial}{\partial\theta_{s,a}}\ratio(\theta_{s,a'})A| A > 0, 1 - \epsilon \leq \ratio(\theta_{s,a'}) \leq \lowerClippingRange[\delta][\KLThree[]]] \\
        & - \mathbb{P}(A < 0, 1 - \epsilon \leq \ratio(\theta_{s,a'}) \leq \lowerClippingRange[\delta][\KLThree[]]) \mathop{\mathbb{E}}_{a' \sim \pi_{\theta_\text{old}}(\cdot|s)}[\frac{\partial}{\partial\theta_{s,a}}\ratio(\theta_{s,a'})A| A < 0, 1 - \epsilon \leq \ratio(\theta_{s,a'}) \leq \lowerClippingRange[\delta][\KLThree[]]] \\
        = & - \mathbb{P}(1 - \epsilon \leq \ratio(\theta_{s,a'}) \leq \lowerClippingRange[\delta][\KLThree[]]) \mathop{\mathbb{E}}_{a' \sim \pi_{\theta_\text{old}}(\cdot|s)}[\frac{\partial}{\partial\theta_{s,a}}\ratio(\theta_{s,a'})A| 1 - \epsilon \leq \ratio(\theta_{s,a'}) \leq \lowerClippingRange[\delta][\KLThree[]]] \\
    \end{aligned}
\end{equation}

Due to we have
\begin{equation}
    \begin{aligned}
        \frac{\partial}{\partial\theta_{s,a}}\ratio(\theta_{s,a'}) = &\mathbb{I}_{\{a=a'\}}\frac{\pi_{\theta^k}(a|s)}{\pi_{\theta_\text{old}}(a'|s)} - \pi_{\theta^k}(a|s)\frac{\pi_{\theta^k}(a'|s)}{\pi_{\theta_\text{old}}(a'|s)} \\
        & \\
    \end{aligned}
\end{equation}

We assume the event
\begin{equation}
    \begin{aligned}
        X_{-}(s)= & \{a\in \mathcal{A}(s)| \ratio(\theta_{s,a}) \in [1 - \epsilon, \lowerClippingRange[\delta][\KLThree[]]], 1 + \epsilon = \upperClippingRange[\delta][\KLThree[]]\} \\
    \end{aligned}
\end{equation}

Then
\begin{equation}
    \begin{aligned}
        & \theta^{\text{ATR}, k+1}_{s, a} - \theta^{\text{ratio,}k+1}_{s, a} \\
        = & - \eta d^{\pi_{\theta_\text{old}}(s)} \mathbb{P}(1 - \epsilon \leq \ratio(\theta_{s,a'}) \leq \lowerClippingRange[\delta][\KLThree[]]) \mathop{\mathbb{E}}_{a' \sim \pi_{\theta_\text{old}}(\cdot|s)}[\frac{\partial}{\partial\theta_{s,a}}\ratio(\theta_{s,a'})A| 1 - \epsilon \leq \ratio(\theta_{s,a'}) \leq \lowerClippingRange[\delta][\KLThree[]]] \\
        = & - \eta d^{\pi_{\theta_\text{old}}(s)} \mathop{\mathbb{E}}_{a' \sim \pi_{\theta_\text{old}}(\cdot|s)}[\frac{\partial}{\partial\theta_{s,a}}\ratio(\theta_{s,a'})A \mathbb{I}_{X_{-}(s)}(a') ] \\
        = & - \eta d^{\pi_{\theta_\text{old}}(s)}
        \sum_{a' \in \mathcal{A}(s)}
        \left[
        (\mathbb{I}_{\{a=a'\}}\pi_{\theta^k}(a|s) - \pi_{\theta^k}(a|s)\pi_{\theta^k}(a'|s))A\mathbb{I}_{X_{-}(s)}(a')
        \right] \\
        = & - \eta d^{\pi_{\theta_\text{old}}(s)}
        \left[
        \pi_{\theta^k}(a|s)A\mathbb{I}_{X_{-}(s)}(a) 
        - 
        \pi_{\theta^k}(a|s)
        \mathop{\mathbb{E}}_{a' \sim \pi_{\theta^k}(\cdot|s)}
        [A\mathbb{I}_{X_{-}(s)}(a')]
        \right] \\
        = & - \eta d^{\pi_{\theta_\text{old}}(s)}
        \pi_{\theta^k}(a|s)
        \left[
        A\mathbb{I}_{X_{-}(s)}(a) 
        - 
        \mathop{\mathbb{E}}_{a' \sim \pi_{\theta^k}(\cdot|s)}
        [A\mathbb{I}_{X_{-}(s)}(a')]
        \right] \\
    \end{aligned}
\end{equation}

Then, we consider the second case: $\lowerClippingRange[\delta][\KLThree[]] = 1 - \epsilon$ and $1 + \epsilon < \upperClippingRange[\delta][\KLThree[]]$.
We assume the event
\begin{equation}
    \begin{aligned}
        X_{+}(s)= & \{a\in \mathcal{A}(s)| \ratio(\theta_{s,a}) \in [1 + \epsilon, \upperClippingRange[\delta][\KLThree[]]], 1 - \epsilon = \lowerClippingRange[\delta][\KLThree[]] \\
    \end{aligned}
\end{equation}

Similarly, in this case, we have
\begin{equation}
    \begin{aligned}
        & \theta^{\text{ATR}, k+1}_{s, a} - \theta^{\text{ratio,}k+1}_{s, a} \\
        = & \eta d^{\pi_{\theta_\text{old}}(s)} \mathbb{P}(1 + \epsilon \leq \ratio(\theta_{s,a'}) \leq \upperClippingRange[\delta][\KLThree[]]) \mathop{\mathbb{E}}_{a' \sim \pi_{\theta_\text{old}}(\cdot|s)}[\frac{\partial}{\partial\theta_{s,a}}\ratio(\theta_{s,a'})A| 1 + \epsilon \leq \ratio(\theta_{s,a'}) \leq \upperClippingRange[\delta][\KLThree[]]]\\
        = & \eta d^{\pi_{\theta_\text{old}}(s)}
        \pi_{\theta^k}(a|s)
        \left[
        A\mathbb{I}_{X_{+}(s)}(a) 
        - 
        \mathop{\mathbb{E}}_{a' \sim \pi_{\theta^k}(\cdot|s)}
        [A\mathbb{I}_{X_{+}(s)}(a')]
        \right] \\
    \end{aligned}
\end{equation}
\end{proof}

\subsection{Entropy Difference Analysis}
\label{appendix:proof:entropy_diff}

\begin{theorem}[Entropy Difference]
Let $\Delta \mathcal{H}:= \mathcal{H}(\theta^{\text{ATR}, k+1} | s) - \mathcal{H}(\theta^{\text{ratio,}k+1} | s)$ denote the entropy difference between the ATR-based and ratio-based clipping methods.

Under event $X_{-}(s)$, we have
$
    \Delta \mathcal{H} = 
    \eta d^{\pi_{\theta_\text{old}}(s)}\mathop{Cov}_{a \sim \pi_{\theta^k}(\cdot|s)}\left(A\mathbb{I}_{X_{-}(s)}(a), \log \pi_{\theta^k}(a|s) \right).
$

Under event $X_{+}(s)$, we have
$
    \Delta \mathcal{H} = 
    - \eta d^{\pi_{\theta_\text{old}}(s)}\mathop{Cov}_{a \sim \pi_{\theta^k}(\cdot|s)}\left(A\mathbb{I}_{X_{+}(s)}(a), \log \pi_{\theta^k}(a|s) \right).
$
\end{theorem}

\begin{proof}
The first-order Taylor expansion of policy entropy
\begin{equation}
    \begin{aligned}
        \mathcal{H}(\theta^{k+1} | s) - \mathcal{H}(\theta^{k} | s) 
        =& - \mathop{\mathbb{E}}_{a \sim \pi_{\theta^k}(\cdot|s)}[(\theta^{k+1}_{s, a} - \theta^{k}_{s, a})(\log \pi_{\theta^k}(a|s) + \mathcal{H}(\theta^{k} | s))] + \mathcal{O}((\Delta \theta)^2) \\
    \end{aligned}
\end{equation}

Consider that we have two different clipping methods with corresponding $\theta^{\text{ratio,}k+1}$ and $\theta^{\text{ATR}, k+1}$
\begin{equation}
    \begin{aligned}
        & \mathcal{H}(\theta^{\text{ATR}, k+1} | s) - \mathcal{H}(\theta^{\text{ratio,}k+1} | s) \\
        = & \left(\mathcal{H}(\theta^{\text{ATR}, k+1} | s) - \mathcal{H}(\theta^{k} | s)\right) - \left(\mathcal{H}(\theta^{ratio, k+1} | s) - \mathcal{H}(\theta^{k} | s) \right)\\
        = & - \mathop{\mathbb{E}}_{a \sim \pi_{\theta^k}(\cdot|s)}[(\theta^{\text{ATR}, k+1}_{s, a} - \theta^{k}_{s, a})(\log \pi_{\theta^k}(a|s) + \mathcal{H}(\theta^{k} | s))]
        + \mathop{\mathbb{E}}_{a \sim \pi_{\theta^k}(\cdot|s)}[(\theta^{\text{ratio,}k+1}_{s, a} - \theta^{k}_{s, a})(\log \pi_{\theta^k}(a|s) + \mathcal{H}(\theta^{k} | s))]\\
        = & - \mathop{\mathbb{E}}_{a \sim \pi_{\theta^k}(\cdot|s)}[(\theta^{\text{ATR}, k+1}_{s, a} - \theta^{\text{ratio,}k+1}_{s, a})(\log \pi_{\theta^k}(a|s) + \mathcal{H}(\theta^{k} | s))]\\
        = & - \mathop{\mathbb{E}}_{a \sim \pi_{\theta^k}(\cdot|s)}[(\theta^{\text{ATR}, k+1}_{s, a} - \theta^{\text{ratio,}k+1}_{s, a})(\log \pi_{\theta^k}(a|s) - \mathop{\mathbb{E}}_{a' \sim \pi_{\theta^k}(\cdot|s)}[\log \pi_{\theta^k}(a'|s)])]\\
    \end{aligned}
\end{equation}

For the event $X_{-}(s)$, we have
\begin{equation}
    \begin{aligned}
        & \mathcal{H}(\theta^{\text{ATR}, k+1} | s) - \mathcal{H}(\theta^{\text{ratio,}k+1} | s) \\
        = & - \mathop{\mathbb{E}}_{a \sim \pi_{\theta^k}(\cdot|s)}
        \left[\left(
        - \eta d^{\pi_{\theta_\text{old}}(s)}
        \pi_{\theta^k}(a|s)
        \left[
        A\mathbb{I}_{X_{-}(s)}(a) 
        - 
        \mathop{\mathbb{E}}_{a' \sim \pi_{\theta^k}(\cdot|s)}
        [A\mathbb{I}_{X_{-}(s)}(a')]
        \right]
        \right)
        \left(\log \pi_{\theta^k}(a|s) - \mathop{\mathbb{E}}_{a' \sim \pi_{\theta^k}(\cdot|s)}[\log \pi_{\theta^k}(a'|s)]
        \right)
        \right]\\
        = & \eta d^{\pi_{\theta_\text{old}}(s)}\mathop{Cov}_{a \sim \pi_{\theta^k}(\cdot|s)}\left(
        A\mathbb{I}_{X_{-}(s)}(a)
        ,
        \log \pi_{\theta^k}(a|s) 
        \right) \\
    \end{aligned}
\end{equation}

For the event $X_{+}(s)$, we have
\begin{equation}
    \begin{aligned}
        & \mathcal{H}(\theta^{\text{ATR}, k+1} | s) - \mathcal{H}(\theta^{\text{ratio,}k+1} | s) \\
        = & - \mathop{\mathbb{E}}_{a \sim \pi_{\theta^k}(\cdot|s)}
        \left[\left(
        \eta d^{\pi_{\theta_\text{old}}(s)}
        \pi_{\theta^k}(a|s)
        \left[
        A\mathbb{I}_{X_{+}(s)}(a) 
        - 
        \mathop{\mathbb{E}}_{a' \sim \pi_{\theta^k}(\cdot|s)}
        [A\mathbb{I}_{X_{+}(s)}(a')]
        \right]
        \right)
        \left(\log \pi_{\theta^k}(a|s) - \mathop{\mathbb{E}}_{a' \sim \pi_{\theta^k}(\cdot|s)}[\log \pi_{\theta^k}(a'|s)]
        \right)
        \right]\\
        = & - \eta d^{\pi_{\theta_\text{old}}(s)}\mathop{Cov}_{a \sim \pi_{\theta^k}(\cdot|s)}\left(
        A\mathbb{I}_{X_{+}(s)}(a)
        ,
        \log \pi_{\theta^k}(a|s) 
        \right) \\
    \end{aligned}
\end{equation}

\end{proof}

\clearpage
\section{Implementation Details}\label{appendix:implementation_detail}
The hyper-parameter setting used in this work is provided in~\Cref{appendix:table:hyper_parameter}.
Following recent works~\citep{hu2025open, liu2025understanding}, we set $\beta=0$ for all methods.
We adopt the official implementation from Unsloth~\citep{unsloth} and TRL~\citep{vonwerra2022trl}.
For all methods, we perform the supervised fine-tuning on the OpenMathReasoning dataset~\citep{moshkov2025aimo2} for formatting before the RLVR process.
On one single NVIDIA A100 GPU, ATR-GRPO requires an average training time of approximately 7 and 16 hours per run for Qwen3-1.7B and Qwen3-8B, respectively.
For the implementation of baselines, we adopt the recommended value settings from the original papers, including DAPO~\citep{yu2025dapo}, DCPO~\citep{yang2025dcpo}, Clip-Cov~\citep{cui2025entropy}, and SAPO~\citep{gao2025soft}.


\begin{table*}[h]
\centering
\caption{Hyper-parameter Setting.\label{appendix:table:hyper_parameter}}
\begin{tabular}{ll}
\hline\hline
Hyper-Parameter             & Value \\ \hline
\multicolumn{2}{c}{Train}           \\
Maximum Sequence Length     & 2048  \\
Lora Rank                   & 32    \\
Lora Alpha                  & 64    \\
Temperature                 & 1.0   \\
Learning Rate               & 5e-6  \\
Weight Decay                & 1e-3  \\
Optimizer                   & AdamW~\citep{loshchilov2017decoupled} \\
Batch Size            & 8     \\
Gradient Accumulation Steps & 4     \\
Top-p                       & 1.0   \\
Top-k                       & -1    \\
Group Size $G$              & 8     \\ 
KL coefficient $\beta$ (reference policy) & 0     \\ 
\hline
\multicolumn{2}{c}{Evaluation}      \\
Temperature                 & 0.3   \\
Max Tokens                  & 32768 \\
Top-p                       & 0.95  \\
Top-k                       & -1    \\
\hline\hline
\end{tabular}
\end{table*}

\clearpage

\end{document}


%% file: lib_custom/A_ALL.tex
\ifshowtmp
    \newcommand{\todom}[1]{{\textcolor{usccardinal}{(todo minor: #1)}}}
	\newcommand{\todo}[1]{{\textcolor{red}{(TODO: #1)}}}

    \newcommand{\tocheck}[1]{{\textcolor{red}{(TO CHECK: #1)}}}

	\newcommand{\toRemove}[1]{ }

	\newcommand{\T}[1]{\textcolor{usccardinal}{#1}}
    \newcommand{\instruct}[1]{}

    \makeatletter
    \def\w#1 {\T{#1}~}
    \makeatother

\else
    \newcommand{\todom}[1]{}
    \newcommand{\tocheck}[1]{}
	
	\newcommand{\T}{}
	\newcommand{\toRemove}[1]{}
	\newcommand{\todo}[1]{}

    \newcommand{\instruct}[1]{}

    \newcommand{\w}{}
\fi

\newcommand{\instructC}[1]{}
\newcommand{\todoLessImportant}[1]{}
\newcommand{\removeForSavingSpace}[1]{}

\newif\ifPPT
\PPTtrue

\newif\ifshowstruct
\showstructtrue
\showstructfalse 

\ifshowstruct
	\newcommand{\tstructz}[1]{[#1] }
\else
	\newcommand{\tstructz}[1]{}
\fi

\label{packages}
\usepackage{mfirstuc}
\usepackage{graphicx}
\usepackage{subcaption}
\usepackage{caption}

\usepackage{multirow}
\usepackage{multicol}
\usepackage{array}
\usepackage{booktabs}

\usepackage{amsmath}
\usepackage{amssymb, graphicx, url, mathtools}
\usepackage[overload]{empheq}

\usepackage{amsthm}

\usepackage{algorithmic,algorithm}   

\usepackage{footnote}
\makesavenoteenv{tabular}
\makesavenoteenv{table}
\usepackage{pdfpages}
\usepackage{textcomp}
\usepackage{chngcntr}
\usepackage{verbatim}

\usepackage{xparse,calc}
\usepackage{afterpage}
\usepackage{comment}

\usepackage{cleveref}

\label{Command}

\counterwithin*{question}{section}

\makeatletter
\def\renewtheorem#1{%
  \expandafter\let\csname#1\endcsname\relax
  \expandafter\let\csname c@#1\endcsname\relax
  \gdef\renewtheorem@envname{#1}
  \renewtheorem@secpar
}
\def\renewtheorem@secpar{\@ifnextchar[{\renewtheorem@numberedlike}{\renewtheorem@nonumberedlike}}
\def\renewtheorem@numberedlike[#1]#2{\newtheorem{\renewtheorem@envname}[#1]{#2}}
\def\renewtheorem@nonumberedlike#1{  
\def\renewtheorem@caption{#1}
\edef\renewtheorem@nowithin{\noexpand\newtheorem{\renewtheorem@envname}{\renewtheorem@caption}}
\renewtheorem@thirdpar
}
\def\renewtheorem@thirdpar{\@ifnextchar[{\renewtheorem@within}{\renewtheorem@nowithin}}
\def\renewtheorem@within[#1]{\renewtheorem@nowithin[#1]}
\makeatother

\renewtheorem{theorem}{Theorem}
\renewtheorem{remark}{Remark}
\renewtheorem{definition}{Definition}
\renewtheorem{example}{Example}

\crefname{mdpexample}{MDP Example}{MDP Examples}

\crefname{todo}{TODO}{TODO}

\makeatletter
\newcommand{\dummylabel}[2]{\def\@currentlabel{#2}\label[td]{#1}}
\makeatother

\usepackage{thmtools} 
\usepackage{thm-restate}

\usepackage{xstring}

\newcolumntype{+}{>{\global\let\currentrowstyle\relax}}
\newcolumntype{Y}{>{\currentrowstyle}}\label{key}
\newcolumntype{-}{>{\currentrowstyle}}

\newcommand{\breakingcomma}{
  \begingroup\lccode`~=`,
  \lowercase{\endgroup\expandafter\def\expandafter~\expandafter{~\penalty0 }}
}

\DeclareFontFamily{U}{mathx}{\hyphenchar\font45}
\DeclareFontShape{U}{mathx}{m}{n}{<-> mathx10}{}
\DeclareSymbolFont{mathx}{U}{mathx}{m}{n}
\DeclareMathAccent{\widebar}{0}{mathx}{"73}

\renewcommand{\algorithmicrequire}{\textbf{Input:}}
\renewcommand{\algorithmicensure}{\textbf{Output:}}

\label{Settings}

\usepackage{makecell}

\crefname{equation}{eq.}{eqs.}
\Crefname{equation}{Eq.}{Eqs.}
\Crefname{Section}{Sec.}{Sec.}

\newcommand{\reftodo}[1]{\todo{cref}}

\newcommand{\Creftodo}[1]{\todo{cref}}
\newcommand{\creftodo}[1]{\todo{cref}}
\newcommand{\citetodo}[1]{\todo{cite}}

\newcommand{\Creftd}[1]{\todo{cref}}
\newcommand{\creftd}[1]{\todo{cref}}
\newcommand{\citetd}[1]{\todo{cite}}

\usepackage[export]{adjustbox}

\usepackage{xcolor}
\definecolor{bluemy}{HTML}{1f77b4}

\usepackage{enumitem}

    \usepackage{xr}
    \makeatletter
    \newcommand*{\addFileDependency}[1]{
      \typeout{(#1)}
      \@addtofilelist{#1}
      \IfFileExists{#1}{}{\typeout{No file #1.}}
    }
    \makeatother

\NewDocumentCommand{\indicatorFunction}{ O{} }{ {\mathbbm{1} }_{\{#1\}} }

\makeatletter
\def\widebreve{\mathpalette\wide@breve}
\def\wide@breve#1#2{\sbox\z@{$#1#2$}%
     \mathop{\vbox{\m@th\ialign{##\crcr
\kern0.08em\brevefill#1{0.8\wd\z@}\crcr\noalign{\nointerlineskip}%
                    $\hss#1#2\hss$\crcr}}}\limits}
\def\brevefill#1#2{$\m@th\sbox\tw@{$#1($}%
  \hss\resizebox{#2}{\wd\tw@}{\rotatebox[origin=c]{90}{\upshape(}}\hss$}
\makeatletter

\usepackage{bbm}

\usepackage{accents}

\usepackage[makeroom]{cancel}

\usepackage{pifont}

\definecolor{blueMy}{HTML}{1f77b4}
\definecolor{orangeMy}{HTML}{ff7f0e}
\definecolor{greenMy}{HTML}{2ca02c}
\definecolor{redMy}{HTML}{d62728}
\definecolor{purpleMy}{HTML}{9467bd}
\definecolor{brownMy}{HTML}{8c564b}
\definecolor{usccardinal}{rgb}{0.6, 0.0, 0.0}

\usepackage{makecell}
